\pdfoutput=1

\documentclass{article}

\usepackage[numbers]{natbib}

\usepackage[final]{neurips_2024}
\usepackage{amsthm}
\usepackage{amsmath}
\usepackage{amssymb}
\usepackage{graphicx}
\usepackage{mathtools}
\usepackage{enumerate}
\usepackage{enumitem}
\usepackage{footnote}
\usepackage{float}
\usepackage{xspace}
\usepackage{multirow}
\usepackage{nicefrac}
\usepackage{wrapfig}
\usepackage{framed}
\usepackage{url}
\usepackage{units}
\usepackage[colorlinks=true, linkcolor=blue, citecolor=blue]{hyperref}
\usepackage{blkarray}
\PassOptionsToPackage{round}{natbib}

\usepackage{tikz}
\usetikzlibrary{arrows,chains,matrix,positioning,scopes}

\let\hat\widehat

\newtheorem{lemma}{Lemma}
\newtheorem{theorem}{Theorem}
\newtheorem{corollary}{Corollary}
\newtheorem{remark}{Remark}

\newtheorem{assumption}{Assumption}

\DeclareMathOperator*{\minimize}{minimize}
\DeclareMathOperator*{\maximize}{maximize}
\DeclareMathOperator*{\subject}{subject~to}
\newcommand{\DefinedAs}[0]{\mathrel{\mathop:}=}

\DeclareMathOperator*{\argmin}{argmin}
\DeclareMathOperator*{\argmax}{argmax}

\newcommand{\norm}[1]{\left\|{#1}\right\|}

\DeclareFontFamily{OMX}{MnSymbolE}{}
\DeclareFontShape{OMX}{MnSymbolE}{m}{n}{
    <-6>  MnSymbolE5
   <6-7>  MnSymbolE6
   <7-8>  MnSymbolE7
   <8-9>  MnSymbolE8
   <9-10> MnSymbolE9
  <10-12> MnSymbolE10
  <12->   MnSymbolE12}{}
\DeclareSymbolFont{mnlargesymbols}{OMX}{MnSymbolE}{m}{n}
\SetSymbolFont{mnlargesymbols}{bold}{OMX}{MnSymbolE}{b}{n}
\DeclareMathDelimiter{\llangle}{\mathopen}{mnlargesymbols}{'164}{mnlargesymbols}{'164}
\DeclareMathDelimiter{\rrangle}{\mathclose}{mnlargesymbols}{'171}{mnlargesymbols}{'171}

\usepackage{prettyref}

\newcommand{\savehyperref}[2]{\texorpdfstring{\hyperref[#1]{#2}}{#2}}
\newrefformat{eq}{\savehyperref{#1}{\textup{(\ref*{#1})}}}
\newrefformat{eqn}{\savehyperref{#1}{Equation~\eqref{#1}}}
\newrefformat{lem}{\savehyperref{#1}{Lemma~\ref*{#1}}}
\newrefformat{def}{\savehyperref{#1}{Definition~\ref*{#1}}}
\newrefformat{line}{\savehyperref{#1}{line~\ref*{#1}}}
\newrefformat{thm}{\savehyperref{#1}{Theorem~\ref*{#1}}}
\newrefformat{corr}{\savehyperref{#1}{Corollary~\ref*{#1}}}
\newrefformat{cor}{\savehyperref{#1}{Corollary~\ref*{#1}}}
\newrefformat{sec}{\savehyperref{#1}{Section~\ref*{#1}}}
\newrefformat{app}{\savehyperref{#1}{Appendix~\ref*{#1}}}
\newrefformat{ap}{\savehyperref{#1}{Appendix~\ref*{#1}}}
\newrefformat{assum}{\savehyperref{#1}{Assumption~\ref*{#1}}}
\newrefformat{as}{\savehyperref{#1}{Assumption~\ref*{#1}}}
\newrefformat{ex}{\savehyperref{#1}{Example~\ref*{#1}}}
\newrefformat{fig}{\savehyperref{#1}{Figure~\ref*{#1}}}
\newrefformat{alg}{\savehyperref{#1}{Algorithm~\ref*{#1}}}
\newrefformat{rem}{\savehyperref{#1}{Remark~\ref*{#1}}}
\newrefformat{conj}{\savehyperref{#1}{Conjecture~\ref*{#1}}}
\newrefformat{prop}{\savehyperref{#1}{Proposition~\ref*{#1}}}
\newrefformat{proto}{\savehyperref{#1}{Protocol~\ref*{#1}}}
\newrefformat{prob}{\savehyperref{#1}{Problem~\ref*{#1}}}
\newrefformat{claim}{\savehyperref{#1}{Claim~\ref*{#1}}}
\newrefformat{que}{\savehyperref{#1}{Question~\ref*{#1}}}
\usepackage{microtype}
\usepackage{graphicx}
\usepackage{subfigure}
\usepackage{booktabs} 
\usepackage[colorinlistoftodos]{todonotes}
\usepackage{multicol}
\usepackage{algorithm, algpseudocode}
\usepackage{float}
\usepackage{bbm}

\usepackage{lipsum}

\usepackage{xcolor}         
\definecolor{mycyan}{RGB}{0,204,204}

\usepackage{multirow}

\usepackage{mwe}
\newcommand{\tabfigure}[2]{\raisebox{-.09\height}{\includegraphics[#1]{#2}}}

\definecolor{myorange}{RGB}{252, 105, 16}

\definecolor{myblue}{RGB}{32, 101, 164}

\newcommand\crule[3][black]{\textcolor{#1}{\rule{#2}{#3}}}

\begin{document}

\title{Constrained Diffusion Models via Dual Training}

\author{%
  Shervin~Khalafi 
  \;\;\;\;
  \;\;\;\;
  \;\;\;\;
  Dongsheng~Ding\thanks{Corresponding author.}
  \;\;\;\;
  \;\;\;\;
  \;\;\;\;
  Alejandro~Ribeiro\\[0.05cm]
  \texttt{\{shervink,dongshed,aribeiro\}@seas.upenn.edu}
  \\[0.15cm]
  University of Pennsylvania
}

\maketitle

\begin{abstract}
     Diffusion models have attained prominence for their ability to synthesize a probability distribution for a given dataset via a diffusion process,  enabling the generation of new data points with high fidelity. However, diffusion processes are prone to generating samples that reflect biases in a training dataset. To address this issue, we develop constrained diffusion models by imposing diffusion constraints based on desired distributions that are informed by requirements. Specifically, we cast the training of diffusion models under requirements as a constrained distribution optimization problem that aims to reduce the distribution difference between original and generated data while obeying constraints on the distribution of generated data. We show that our constrained diffusion models generate new data from a mixture data distribution that achieves the optimal trade-off among objective and constraints. To train constrained diffusion models, we develop a dual training algorithm and characterize the optimality of the trained constrained diffusion model. We empirically demonstrate the effectiveness of our constrained models in two constrained generation tasks: (i) we consider a dataset with one or more underrepresented classes where we train the model with constraints to ensure fairly sampling from all classes during inference; (ii) we fine-tune a pre-trained diffusion model to sample from a new dataset while avoiding overfitting.

    
\end{abstract}

\section{Introduction}\label{sec: introduction}

Diffusion models have become a driving force of modern generative modeling, achieving ground-breaking performance in tasks ranging from image/video/audio generation \cite{saharia2022photorealistic,blattmann2023align,kong2020diffwave} 
to molecular design for drug discovery \cite{weiss2023guided,watson2023novo}. Diffusion models learn diffusion processes that produce a probability distribution for a given dataset (e.g., images) from which we can generate new data (e.g., classic models \cite{sohl2015deep,ho2020denoising,song2021denoising,yang2023survey}). As diffusion models are used to generate data with societal impacts, e.g., art generation and content creation for media, they must comply with requirements from specific domains, e.g., social fairness in image generation~\cite{luccioni2023stable,naik2023social}, aesthetic properties of images~\cite{chen2024integrating,chen2024towards}, bioactivity in molecule generation~\cite{huang2023dual}, and more  \cite{kazerouni2023diffusion,zhang2023audio,cao2024survey,croitoru2023diffusion,yang2023diffusion,chen2024overview}.


Classic diffusion models \cite{sohl2015deep,ho2020denoising,song2021denoising} have been extended to generate data under different requirements through either first-principle methods or fine-tuning. In image generation with fairness, for instance, first principle methods directly mitigate biases towards social/ethical identities by revising the training loss functions per biased/unbiased data~\cite{liu2022compositional,kim2024training}; while fine-tuning methods align biased diffusion models with desired data distributions by optimizing the associated metrics~\cite{fan2024reinforcement,uehara2024bridging,uehara2024fine,tang2024fine}. Although these methods allow us to equip diffusion models with specific requirements, they are often designed for particular generation tasks and do not provide transparency on how these requirements are satisfied. Since diffusion models are trained by minimizing the loss functions of diffusion processes, it is natural to incorporate requirements into diffusion models by imposing constraints on these optimization problems. Therefore, it is imperative to develop diffusion models under constraints by generalizing constrained learning methods and theory (e.g.,~\cite{chamon2020probably,chamon2022constrained,elenter2024near,hounie2024resilient}) for diffusion models.  

In this work, we formulate the training of diffusion models under requirements as a constrained distribution optimization problem in which the objective function measures a training loss induced by the original data distribution, and the constraints require other training losses induced by some desirable data distributions to be small. This constrained formulation can be instantiated for several critical requirements. For instance, to promote fairness for unrepresented groups, the constraints can encode the closeness of the model to the distributions of underrepresented data. Compared with the typical importance re-weighting method~\cite{kim2024training}, our constrained formulation provides an optimal trade-off between matching given data distribution and following reference distribution. Not limited to constraints with other desirable data distributions, our constrained formulation captures more general requirements. For instance, when adapting a pretrained diffusion model to new data, the constraints require the model to be close to the pretrained model, not degrading the original generation ability.
Specifically, our main contribution is three-fold.
\begin{itemize}
	\item[(i)] We propose and analyze constrained diffusion models from the perspective of constrained optimization in an infinite-dimensional distribution space, where the constraints require KL divergences between the model and our desired data distributions to be under some thresholds. We exploit the strong duality in convex optimization to show that our constrained diffusion models generate new data from a mixture data distribution that achieves the optimal trade-off among objective and constraints. 
	\item[(ii)] To train constrained diffusion models, we introduce parametrized constrained diffusion models and develop a Lagrangian-based dual training algorithm. We exploit the relation between un/parametrized problems to show that constrained diffusion models generate new data from the optimal mixture distribution, up to some optimization/parametrization errors.
	\item[(iii)] We empirically demonstrate the merit of our constrained diffusion models in two aforementioned requirements. In the fair generation task, we show that our constrained model promotes sampling more from the minority classes compared to unconstrained models, leading to fair sampling across all classes. In the adaptation task, we show that our fine-tuned constrained model learns to generate new data without significantly degrading the original generation ability, compared to the unconstrained model which tends to overfit to new data. 
\end{itemize}

\noindent\textbf{Related work.} 
As a first-principle method, our constrained diffusion approach is more relevant to diffusion models that incorporate requirements in distribution space~\cite{dhariwal2021diffusion,ho2021classifier,liu2022compositional,du2023reduce,power2023sampling,friedrich2023fair}, rather than those applied in sample space~\cite{huang2022riemannian,lou2023reflected,giannone2023aligning,fishman2024metropolis,fishman2024diffusion,christopher2024projected,liu2024mirror,feng2024neural}. In comparison with conditional diffusion models that restrict generation through conditional information~\cite{dhariwal2021diffusion,ho2021classifier,bansal2023universal}, our constrained diffusion models impose distribution constraints within the constrained optimization framework. Compared to compositional generation~\cite{liu2022compositional,du2023reduce,power2023sampling} and fair diffusion~\cite{friedrich2023fair}, our work provides a constrained learning approach to balance different distribution models using Lagrangian multipliers, which is different from equal weights~\cite{liu2022compositional,du2023reduce}, hyperparameter~\cite{power2023sampling} or fair
guidance~\cite{friedrich2023fair}. Our constrained approach is also relevant to the importance re-weighting method for diffusion models~\cite{kim2024training} and GANs~\cite{choi2020fair}, which reduces bias in a dataset by re-weighting it with a pre-trained debiased density ratio. In contrast, we design our diffusion models to mitigate the bias in a dataset by imposing distribution constraints without pre-training. In addition to being distinct from existing methods, we provide a systematic study of our constrained diffusion models, covering duality analysis, dual-based algorithm design, and convergence analysis, which is absent in the diffusion model literature. 

Our work is also pertinent to recently surging fine-tuning and alignment methods that aim to improve pre-trained diffusion models by optimizing their downstream performance, e.g., aesthetic scores of images. In reward-guided fine-tuning methods, such as supervised learning~\cite{lee2023aligning,yuan2023reward,wu2023human}, control-based feedback learning~\cite{xu2023imagereward,shen2023finetuning,uehara2024fine,tang2024fine,clark2024directly}, and reinforcement learning~\cite{fan2023optimizing,hao2023optimizing,fan2024reinforcement,uehara2024bridging,black2024training,zhang2024confronting}, reward functions have to be pre-trained from an authentic evaluation dataset, and the trade-off for reward functions in pre-trained diffusion models is often 
regulated heuristically. In contrast, our constrained approach directly minimizes the gap between a fine-tuning model and a high-quality dataset with the desired properties, while ensuring the generated outputs being close to that of pre-trained models. 

Compared to other generative models under requirements (e.g., VAEs~\cite{rezende2018generalized}, GANs~\cite{choi2020fair}), and classical sampling methods (e.g., Langevin dynamics and Stein variational gradient descent~\cite{liu2021sampling}), our work is different because we focus on diffusion-based generative models.



\section{Preliminaries}\label{sec:preliminaries}

We overview diffusion models from the perspective of variational inference~\cite{luo2022understanding} by presenting forward/backward processes in~Section~\ref{subsec:forward backward}, and the evidence lower bound in Section~\ref{subsec:elbo}.

\subsection{Forward and backward processes}\label{subsec:forward backward}

The forward process begins with a true data sample $x_0 \in \mathbb{R}^d$, and generates latent variables $\{x_t\}_{t\,=\,1}^T$ from $t=1$ to $T$ by adding white noise recursively. The joint distribution of latent variables results from conditional probabilities 
$
q(x_{1:T}\,\vert\,x_0)  = 
\prod_{t\,=\,1}^T q(x_t\,\vert\,x_{t-1})
$,
where the distribution of latent variable $x_t$ conditioned on the previous latent $x_{t-1}$ is given by a Gaussian distribution $q(x_t\,\vert\,x_{t-1}) 
\; =  \;
\mathcal{N}(x_t; \mu_t, \sigma_q^2(t) I)$, 
where $\mu_t = \sqrt{\alpha_t} \,x_{t-1}$ and $\sigma_q^2(t) = 1-\alpha_t$. Since the forward process is a linear Gaussian model with pre-selected mean and variance, it is solely determined by the data distribution $q(x_0)$. We often refer to $q(x_0)$ as a forward process.     

The backward process begins with the latent $x_T$ sampled from the standard Gaussian $p(x_T) = \mathcal{N}(x_T; 0, I)$, and decodes latent variables from $t=T$ to $t=0$ with a joint distribution
\begin{equation}\label{eq:joint backward}
	p(x_{0:T}) 
	\; = \;
	p(x_T)  \prod_{t\,=\,1}^T p(x_{t-1}\,\vert\,x_t).
\end{equation}
Here, $p$ is our distribution model that can be used to generate new samples.  
We denote by $\mathcal{P}$ the set of all joint distributions over $x_{0:T}$ in form of~\eqref{eq:joint backward}, where $p(x_{t-1}\,\vert\,x_t)$ is a conditional Gaussian with a fixed variance (see Appendix~\ref{app:ELBO}). Throughout the paper, we work in the convergent regimes of the backward process (e.g.,~\cite{chen2023sampling,chen2023improved,benton2023linear}). Without loss of generality, we adopt the convergent regime in~\cite{li2023towards} by taking the scheduling parameter $\alpha_1 = 1-{1}/{T^{c_0}}$ and $\alpha_t = 1 - c_T \min \left( (1 - \alpha_1) (1 + c_T)^t, 1\right)$ for $t>1$, and the variance as $\sigma_p^2(t) = 1/\alpha_t - 1$, where $c_T \DefinedAs c_1 \log (T)/ T$ and $c_0$, $c_1$ are some constants. Hence, $\bar{\alpha}_T \approx 0$ implies $q(x_T) \simeq \mathcal{N}(x_T; 0, I)$. Thus, $q(x_{0:T}) \in \mathcal{P}$. Also, $\bar{\alpha}_1 \approx 1$ implies $q(x_1) \simeq q(x_0)$. It is ready to generalize our results to other diffusion processes (e.g.,~\cite{li2024accelerating,huang2024convergence,benton2024nearly}).  

\subsection{The evidence lower bound (ELBO) }\label{subsec:elbo}

Denote the KL divergence of distribution $q$ from distribution $p$ by $D_{\text{KL}}(q\,\Vert\,p) \DefinedAs \mathbb{E}_{x\sim q(x)} \log (\frac{q(x)}{p(x)})$. 
Generative diffusion modeling aims to generate samples whose distribution is close to that of an observed dataset of samples.
Formally, we express this objective as maximizing the log-likelihood of an observation generated by the diffusion model: $\maximize_{p \,\in\, \mathcal{P}} \mathbb{E}_{q(x_0)} [ \,\log p(x_0)\, ]$, where 
\begin{equation}\label{eq:log-likelihood}
	\mathbb{E}_{q(x_0)} 
	\left[\,
	\log p(x_0) 
	\,\right]
	\; = \;
	E(p;q)
	\,+\,
	\mathbb{E}_{q(x_0)}  
	\left[\,
	D_{\text{KL}}\left(q(x_{1:T}\,\vert\,x_0)\,\Vert\, p(x_{1:T}\,\vert\,x_0)\right) 
	\,\right]
\end{equation}
and $E(p; q) \DefinedAs \mathbb{E}_{q(x_0)} \mathbb{E}_{q(x_{1:T}\,\vert\,x_0)} \log \frac{p(x_{0:T})}{q(x_{1:T}\,\vert\,x_0)}$ is known as ELBO in variational inference~\cite{blei2017variational,luo2022understanding}. Alternatively, we aim to minimize the KL divergence between the forward/backward processes, 
\begin{equation}\label{eq:KL forward/backward}
	D_{\text{KL}} \left(q(x_{0:T})\,\Vert\, p(x_{0:T})\right) 
	\; = \;
	-\, E(p; q)
	\, + \,
	\mathbb{E}_{q(x_0)}
	\left[\,
	\log q(x_0) 
	\,\right].
\end{equation}
Thus, we connect the log-likelihood maximization to the KL divergence minimization via ELBO. 
\begin{lemma}[Equivalent formulations]\label{lem:equivalence}
	The ELBO maximization  and the KL divergence minimization are equivalent over the distribution space $\mathcal{P}$, and the unique solution of these two problems is a solution for the log-likelihood maximization problem,  i.e.,
	\[
	\maximize_{p\,\in\,\mathcal{P}}\; E(p; q)
	\;\; 
	\Leftrightarrow  
	\;\;   \minimize_{p\,\in\, \mathcal{P}}\; D_{\text{KL}} \left(q(x_{0:T})\,\Vert\, p(x_{0:T})\right) 
	\;\;
	\Rightarrow
	\;\;
	\maximize_{p \,\in\, \mathcal{P}}\;
	\mathbb{E}_{q(x_0)} [\,\log p(x_0)\, ] 
	\]

\end{lemma}
\noindent See Appendix~\ref{app:equivalence} for proof. Lemma~\ref{lem:equivalence} states that improving the ELBO score increases the likelihood of a backward process that generates the data, together with the fit of a backward process to the forward process. Hence, finding the best backward process becomes optimizing one of three equivalent objectives. In practice, ELBO serves as a loss function approximated by 
\begin{equation}\label{eq:ELBO expansion}
	\displaystyle
	E(p; q ) 
	\; \approx \; -\,
	\sum_{t\,=\,2}^T \mathbb{E}_{q(x_0)} \mathbb{E}_{q(x_t\,\vert\,x_0)} 
	\left[\,
	D_{\text{KL}}
	\left(q(x_{t-1}\,\vert\,x_t,x_0)\,\Vert\, p (x_{t-1}\,\vert\,x_t)\right)
	\,\right]. 
\end{equation} 
With the variance schedule described in Section~\ref{subsec:forward backward}, it is known that this approximation is almost exact (see Appendix~\ref{app:ELBO} and also~\cite{kingma2023understanding}), which is our focal setting. Using standard diffusion derivations~\cite{luo2022understanding}, the ELBO maximization can be shown to equal to a quadratic loss minimization,
\begin{equation}\label{eq:prediction problem score matching}
	\minimize_{\hat{s} \, \in\, \mathcal{S}} \;\;  
	\mathbb{E}_{x_0, \,t, \,x_t}
	\left[ \,
	\norm{\hat{s}(x_t, t) - \nabla \log q(x_t)}^2
	\,\right] 
\end{equation}
where $\mathbb{E}_{x_0, t, x_t}$ is an expectation over the data distribution $q(x_0)$, a discrete distribution $p_\omega(t)$  from $2$ to $T$, and the forward process $q(x_t\,\vert\,x_0)$ at time $t$ given the data sample $x_0$; see Appendix~\ref{app:ELBO} for details. The minimization is done to find a function $\hat{s} \, \in\, \mathcal{S}$ that can best predict the gradient of the forward process over data $\nabla \log q(x_t)$, commonly called the (Stein) score function, where $\mathcal{S}$ is a set of valid score functions mapping from $\mathbb{R}^d\times \mathbb{N}$ to $\mathbb{R}^d$. In practice, however, we parametrize the estimator $\hat{s}(x_t,t)$ as $\hat{s}_\theta(x_t,t)$ with parameter $\theta$, which gives our focal objective of generative modeling:
$\minimize_{\theta\,\in\,\Theta} 
\mathbb{E}_{x_0,\, t,\, x_t}
\big[ \, 
\norm{\hat{s}_\theta(x_t, t) - \nabla \log q(x_t)}^2
\, \big]$. A parametrized form of $p(x_{t-1}\,\vert\,x_t)$ associated with $\hat{s}_\theta(x_t, t)$ is denoted by $p_\theta(x_{t-1}\,\vert\,x_t)$ and the backward process has a parametrized joint distribution $p_\theta(x_{0:T})$. We remark that the prediction problem~\eqref{eq:prediction problem score matching} can be also be formulated as data or noise prediction instead~\cite{luo2022understanding}, with our results directly transferable to these formulations.

	

\section{Variational constrained diffusion models}\label{sec:constrained diffusion models}
We introduce constrained diffusion models by considering the unparametrized set of joint distributions in Section~\ref{subsec:unparametrized case}, and illustrating constraints via two examples in Section~\ref{subsec:fair image generation}.

\subsection{KL divergence-constrained diffusion model: unparametrized case}\label{subsec:unparametrized case}

The standard diffusion model specifies a single data distribution, denoted by $q$ in Lemma~\ref{lem:equivalence}. To account for other generation requirements, we introduce $m$ additional data distributions $\{q^i\}_{i\,=\,1}^m$ that represent $m$ desired properties on generated data. To incorporate new properties into the diffusion model, we formulate an unparametrized KL divergence-constrained optimization problem, 
\begin{equation}\label{eq:KL-constrained unparameterized optimization}
	\begin{array}{rl}
		\displaystyle\minimize_{p \,\in\, \mathcal{P}} & D_{\text{KL}}\left(q(x_{0:T})\,\Vert\, p(x_{0:T})\right)
		\\[0.2cm]
		\subject &  D_{\text{KL}}\left(q^i(x_{0:T})\,\Vert\, p(x_{0:T})\right)\; \leq \; b_i \;\; \text{ for } i = 1, \ldots, m.
	\end{array}
	\tag{U-KL} 
\end{equation}
Let an optimal solution to Problem~\eqref{eq:KL-constrained unparameterized optimization} be $p^\star$. Then the optimal value of the objective function is $F^\star \DefinedAs D_{\text{KL}}(q\,\Vert\,  p^\star)$.
Problem~\eqref{eq:KL-constrained unparameterized optimization} aims to find a model $p^\star$ that generates data from the original distribution $q$ while staying close to $m$ distributions $\{ q^i \}_{i\,=\,1}^m$ that encode our desired properties, e.g., unbiasedness towards minorities. Let the Lagrangian for Problem~\eqref{eq:KL-constrained unparameterized optimization} be 
\begin{equation}\label{eq:Lagrangian of KL-constrained unparameterized optimization}
	\mathcal{L}(p, \lambda) \; = \;
	D_{\text{KL}}\left(q(x_{0:T})\,\Vert\,  p(x_{0:T})\right) 
	\,+\,
	\sum_{i\,=\,1}^m \lambda_i 
	\left(\,
	D_{\text{KL}}\left(q^i(x_{0:T})\,\Vert\,  p(x_{0:T})\right) - b_i
	\,\right)
\end{equation}
for $\lambda \geq 0$. 
The dual function $g(\lambda)$ is given by $g(\lambda) \DefinedAs \min_{p\,\in\,\mathcal{P}} \mathcal{L}(p, \lambda)$, which is always concave. 

To make Problem~\eqref{eq:KL-constrained unparameterized optimization} meaningful, we assume the
constraints are strictly satisfied by some model.
\begin{assumption}[Strict feasibility]\label{ass:feasibility_KL}
	There exists a model $p \in \mathcal{P}$ and $\zeta>0$ such that $D_{\text{\normalfont KL}}(q^i(x_{0:T})\,\Vert\, p(x_{0:T})) \leq b_i - \zeta$ for all $i = 1, \ldots, m$.   
\end{assumption}

Let an optimal dual variable of Problem~\eqref{eq:KL-constrained unparameterized optimization} be $\lambda^\star \in \argmax_{\lambda\,\geq\,0} g(\lambda)$, and the optimal value of the dual function be $D^\star \DefinedAs  g(\lambda^\star)$. From weak duality, the duality gap is non-negative, i.e., $F^\star - D^\star \geq 0$. Moreover, 
due to the convexity of KL divergence, Problem~\eqref{eq:KL-constrained unparameterized optimization} is a convex optimization problem, and thus it satisfies strong duality; see Appendix~\ref{app:duality} for proof.

\begin{lemma}[Strong duality]\label{lem:duality}
	Let Assumption~\ref{ass:feasibility_KL} hold. Then, Problem~\eqref{eq:KL-constrained unparameterized optimization} has zero duality gap, i.e., $F^\star = D^\star$. Moreover, $(p^\star, \lambda^\star)$ is an optimal primal-dual pair of Problem~\eqref{eq:KL-constrained unparameterized optimization}.
\end{lemma}

Let a mixture data distribution be $q_{\text{mix}}^{(\lambda)} \DefinedAs \left( \, q + \sum_{i\,=\,1}^m\lambda^{i} q^i \, \right)/ (1+\lambda^\top 1)$ for $\lambda \geq 0$. We denote by $q_{\text{mix}}^{(\lambda)}(x_{0:T})$ a joint distribution of the forward process with data distribution $q_{\text{mix}}^{(\lambda)}$.  Leveraging strong duality, we show that an optimal model can  be obtained by solving an equivalent unconstrained problem in Theorem~\ref{thm:solution of unparameterized constrained optimization} and its proof is deferred to Appendix~\ref{app:solution of unparameterized constrained optimization}.           

\begin{theorem}[Optimal constrained model]\label{thm:solution of unparameterized constrained optimization}
	Let Assumption~\ref{ass:feasibility_KL} hold. Then, Problem~\eqref{eq:KL-constrained unparameterized optimization} equals 
	\begin{equation}\label{eq:dual of KL-constrained unparameterized optimization}
		\displaystyle
		\minimize_{p\, \in \,
			\mathcal{P}}
		\;\; D_{\text{\normalfont KL}}\left( q_{\text{\normalfont mix}}^{(\lambda^\star)}(x_{0:T})\,\Vert\,  p(x_{0:T}) \right)
		\tag{U-MIX}
	\end{equation}
	where 
	$q_{\text{\normalfont mix}}^{(\lambda^\star)}(x_{0:T})$ is the joint distribution of the forward process at an optimal dual variable $\lambda^\star$. 
\end{theorem}

Theorem~\ref{thm:solution of unparameterized constrained optimization} states that the KL divergence-constrained problem reduces to an unconstrained KL divergence minimization problem. We notice that the KL divergence is zero if and only if two probability distributions match each other. Hence, $q_{\text{\normalfont mix}}^{(\lambda^\star)}(x_{0:T})$ is the optimal solution to Problem~\eqref{eq:dual of KL-constrained unparameterized optimization}. 

\begin{corollary}\label{cor:qmix is solution}
	Let Assumption~\ref{ass:feasibility_KL} hold. Then, the solution of Problem~\eqref{eq:dual of KL-constrained unparameterized optimization}, i.e.,  $p^\star(x_{0:T})
	= q_{\text{\normalfont mix}}^{(\lambda^\star)}(x_{0:T})$, is the solution of Problem~\eqref{eq:KL-constrained unparameterized optimization}.
\end{corollary}
Let $\bar b^i \DefinedAs b^i-\mathbb{E}_{q^i(x_0)}[ \,\log q^i(x_0)\, ]$.
Application of Equality~\eqref{eq:KL forward/backward} to Problem~\eqref{eq:KL-constrained unparameterized optimization} yields an ELBO-based constrained optimization problem, 
\begin{equation}\label{eq:ELBO-constrained unparameterized optimization}
	\begin{array}{rl}
		\displaystyle\minimize_{p \,\in\, \mathcal{P}} & -\,E(p;q)
		\\[0.2cm]
		\subject &  -\,E(p; q^i) 
		\;\leq\; 
		\bar b^i \;\; \text{ for } i = 1, \ldots, m.
	\end{array}
	\tag{U-ELBO} 
\end{equation}
Recall the model representation in Section~\ref{subsec:elbo}, we can characterize each joint distribution $p \in \mathcal{P}$ with a function $\hat{s} \in \mathcal{S}$. Moreover, ELBO reduces to the denoising matching term that has a simplified  quadratic form given in Section~\ref{subsec:elbo}. With this reformulation in mind, we cast Problem~\eqref{eq:ELBO-constrained unparameterized optimization} into a convex optimization problem over the function space $\mathcal{S}$,
\begin{equation}\label{eq:KL-constrained likelihood maximization unparametrization Elbo MSE}
	\begin{array}{rl}
		\displaystyle\minimize_{\hat{s} \, \in\, \mathcal{S}} &  \mathbb{E}_{q(x_0),\, t,\, x_t}
		\left[\,
		\norm{\hat{s}(x_t, t) - \nabla \log q(x_t)}^2
		\,\right]
		\\[0.2cm]
		\subject &  \mathbb{E}_{q^i(x_0),\, t,\, x_t}
		\left[\,
		\norm{\hat{s}(x_t, t) - \nabla \log q^i(x_t)}^2
		\,\right]\; \leq \; \Tilde{b}^i \;\; \text{ for } i = 1, \ldots, m
	\end{array}
	\tag{U-LOSS}
\end{equation}
where $\Tilde{b}^i  \DefinedAs (\bar{b}^i-v)/\bar{\omega}$. Here, the notation $v$ is a constant shift due to the variance mismatch term; see it in Appendix~\ref{app:ELBO}.  We note that scaling or shifting objective and constraints from both sides with some constants doesn't alter the solution to a constrained optimization problem. Thus, the key difference between Problems~\eqref{eq:KL-constrained unparameterized optimization} and~\eqref{eq:KL-constrained likelihood maximization unparametrization Elbo MSE} is the optimization variable (respectively, $p$ and $\hat{s}$). Let the Lagrangian $\mathcal{L}_s(\hat{s}, \lambda)$ for Problem~\eqref{eq:KL-constrained likelihood maximization unparametrization Elbo MSE} be
\[
\displaystyle
\mathbb{E}_{q(x_0),\, t,\, x_t}
\left[
\norm{\hat{s}(x_t, t) - \nabla\log q(x_t)}^2
\right]
\, + \, 
\sum_{i\,=\,1}^m \lambda_i \left(
\mathbb{E}_{q^i(x_0),\, t,\, x_t}
\left[
\norm{\hat{s}(x_t, t) - \nabla\log q^i(x_t)}^2
\right] - \Tilde{b}^i
\right).
\]
Let the associated dual function be $g_s(\lambda) \DefinedAs \min_{\hat{s}\,\in\,\mathcal{S}} \mathcal{L}_s(\hat{s}, \lambda)$ for $\lambda \geq 0$. Hence, $g(\lambda)$ and $g_s(\lambda)$ have the same maximizer $\lambda^\star$, and the partial minimizer $\hat{s}^\star  = \argmin_{\hat{s}\,\in\,\mathcal{S}} \mathcal{L}_s(\hat{s},\lambda^\star)$ is the solution to Problem~\eqref{eq:KL-constrained likelihood maximization unparametrization Elbo MSE}. 
Hence, an optimal primal-dual pair $(\hat{s}^\star,\lambda^\star)$ to Problem~\eqref{eq:KL-constrained likelihood maximization unparametrization Elbo MSE} gives an optimal primal-dual pair $(p^\star,\lambda^\star)$ for Problem~\eqref{eq:KL-constrained unparameterized optimization}, where $p^\star$ is a joint distribution of the backward process induced by $\hat{s}^\star$. By this dual property, we take a dual perspective to train constrained diffusion models: we maximize the dual function $g_s(\lambda)$ to obtain the optimal dual variable $\lambda^\star$, and then recover the solution $\hat{s}^\star$ by minimizing the Lagrangian $\mathcal{L}_s(\hat{s},\lambda^\star)$. 

\subsection{Examples of KL divergence constraints}\label{subsec:fair image generation}

To illustrate our KL-divergence constraints, we provide two generation tasks of exemplary.
\begin{itemize}
	\item[(i)] \textbf{Fairness to underrepresented classes.} We consider a fair image generation task in which some classes are underrepresented in the available training dataset. An example of this would be the Celeb-A dataset~\cite{liu2018large} which contains pictures of celebrity faces with those labeled as male being underrepresented (42\% Male vs 58\% Female). To promote representation of the under-represented classes, we can pose it as an instance of Problem~\eqref{eq:KL-constrained unparameterized optimization}, where each $q^i$ denotes the distribution of an under-represented subset or minority class of $q$.
	\item[(ii)] \textbf{Adapting pretrained model to new data.} Given a pretrained diffusion model over some original dataset that is no longer accessible, we aim to fine-tune the pretrained model for generating data from a new data distribution. Similarly, we can pose this as an instance of Problem~\eqref{eq:KL-constrained unparameterized optimization}, where $q^1$ denotes the distribution of the new data and $q$ is the distribution of samples generated by the pre-trained model.
\end{itemize}

Let $h_i  \DefinedAs -\mathbb{E}_{q^i(x_0)}\left[\,\log q^i(x_0) \,\right]$ be the differential entropy of data distribution $q^i$.  We relate the KL divergence constraints with the optimal dual variable through entropy in Theorem~\ref{thm: characterize dual solution}. 
\begin{theorem}\label{thm: characterize dual solution}
	Let Assumption~\ref{ass:feasibility_KL} hold, and the supports of data distributions $q$ and $\{q^i\}_{i\,=\,1}^m$ be disjoint. Then, the optimal dual variables $\lambda^\star$ to Problem~\eqref{eq:ELBO-constrained unparameterized optimization} are given by
	\[
	\displaystyle
	\frac{\lambda^\star_i}{1 \,+\, (\lambda^\star)^T \mathbf{1}} 
	\; = \;
	{\rm e}^{h_i \,-\, \bar b_i} \;\; \text{\normalfont for } i = 1,\ldots, m. 
	\]
\end{theorem}
See Appendix~\ref{app:more theory} for proof.  Theorem~\ref{thm: characterize dual solution} characterizes the mixture weights in the target distribution $q_{\text{\normalfont mix}}^{(\lambda^\star)}$: (i) the tighter a constraint is (i.e., smaller threshold $\bar b_i$), the more the model will sample from the associated distribution; (ii) for the same constraint thresholds, the model will sample more often from the associated distributions that have higher entropy $h_i$. Assumption~\ref{ass:feasibility_KL} can be relaxed to a feasibility condition for Problem~\eqref{eq:KL-constrained unparameterized optimization}; see Lemma~\ref{lem:feasibility} in Appendix~\ref{app:more theory}.

\section{Parametrization and dual training algorithm }\label{sec:dual training algorithm}

Having introduced unparametrized models, we move to parametrization for constrained diffusion models in Section~\ref{subsec:parametrized case}, provide optimality analysis of a Lagrangian-based dual method in  Section~\ref{subsec:optimality analysis}, and present a practical dual training algorithm in Section~\ref{subsec:primal-dual training}.  

\subsection{KL divergence-constrained diffusion model: parametrized case}\label{subsec:parametrized case}

With the parametrized model $p_\theta$ for $\theta\in\Theta$, we present a parameterized constrained problem, 
\begin{equation}\label{eq:KL-constrained likelihood maximization parametrization}
	\begin{array}{rl}
		\displaystyle\minimize_{\theta \,\in\, \Theta} & D_{\text{KL}}\left(q (x_{0:T}) \,\Vert\, p_\theta (x_{0:T})\right)
		\\[0.2cm]
		\subject &  D_{\text{KL}}\left(q^i(x_{0:T}) \,\Vert\, p_\theta(x_{0:T})\right)\; \leq \; b^i \;\; \text{ for } i = 1, \ldots, m.
	\end{array}
	\tag{P-KL}
\end{equation} 
Let the Lagrangian for Problem~\eqref{eq:KL-constrained likelihood maximization parametrization} be $\bar{\mathcal{L}}(\theta, \lambda) \DefinedAs \mathcal{L}(p_\theta, \lambda)$.
The associated dual function $\bar{g}(\lambda)$ is given by $\bar{g}(\lambda) \DefinedAs \min_{\theta\,\in\,\Theta} \bar{\mathcal{L}}(\theta, \lambda)$. Let an optimal solution to Problem~\eqref{eq:KL-constrained likelihood maximization parametrization} be $\theta^\star$. We denote $\bar p \DefinedAs p_{\theta}$ and $\bar p^\star \DefinedAs p_{\theta^\star}$, and the optimal objective by $\bar{F}^\star \DefinedAs D_{\text{KL}}(q\,\Vert\,\bar p^{\star})$. Let an optimal dual variable be $\bar\lambda^\star \in \argmax_{\lambda\,\geq\,0} \bar{g}(\lambda)$ and the optimal value of the dual function be $\bar{D}^\star \DefinedAs \bar{g}(\bar\lambda^\star)$.

Problem~\eqref{eq:KL-constrained likelihood maximization parametrization} is non-convex in parameter space, and strong duality does not hold any more. Thus, unparametrized results in Section~\ref{subsec:unparametrized case} don't directly apply to Problem~\eqref{eq:KL-constrained likelihood maximization parametrization}. For instance, it's invalid to find an optimal solution via an  unconstrained problem as in Theorem~\ref{thm:solution of unparameterized constrained optimization}, i.e., $\bar{p}^\star(\bar\lambda^\star)\in \argmin_{\theta\,\in\,\Theta} \bar{\mathcal{L}}(\theta,\bar\lambda^\star)$ doesn't equal $\bar{p}^\star$. The effect of parametrization has to be characterized.
However, regardless of parametrization, weak duality always holds, i.e., $\bar{F}^\star - \bar{D}^\star \geq 0$.

To quantify the optimality of $\bar{p}^\star(\bar\lambda^\star)$ (closeness of it to $q_{\text{\normalfont mix}}^{\star}$), we study a practical representation of model $p_\theta$ as a parametrized function $\hat{s}_\theta \in \mathcal{S}_\theta$, where $\mathcal{S}_\theta$ is the set of all parametrized score functions. Problem~\eqref{eq:KL-constrained likelihood maximization unparametrization Elbo MSE} is in a parametrized form of
\begin{equation}\label{eq:KL-constrained likelihood maximization parametrization Elbo MSE}
	\begin{array}{rl}
		\displaystyle\minimize_{\theta} &  \mathbb{E}_{q(x_0),\, t,\, x_t}
		\left[\, 
		\norm{\hat{s}_\theta(x_t, t) - \nabla\log q(x_t)}^2
		\,\right]\;
		\\[0.2cm]
		\subject &  \mathbb{E}_{q^i(x_0),\, t,\, x_t}
		\left[\, 
		\norm{\hat{s}_\theta(x_t, t) - \nabla\log q^i(x_t)}^2
		\,\right]\; \leq \; \Tilde{b}^i \;\; \text{ for } i = 1, \ldots, m.
	\end{array}
	\tag{P-LOSS}
\end{equation}
where $\Tilde{b}^i \DefinedAs (\bar{b}^i-v)/\bar{\omega}$. We note that Problem~\eqref{eq:KL-constrained likelihood maximization parametrization Elbo MSE} is equivalent to Problem~\eqref{eq:KL-constrained likelihood maximization parametrization}. Thus, we let the Lagrangian of Problem~\eqref{eq:KL-constrained likelihood maximization parametrization Elbo MSE} be $\bar{\mathcal{L}}_s(\theta,\lambda) \DefinedAs \mathcal{L}_s(\hat{s}_\theta,\lambda)$, and the dual function $\bar{g}_s(\lambda) \DefinedAs \minimize_{\theta\,\in\,\Theta} \bar{\mathcal{L}}(\theta,\lambda)$. Since $\bar{g}(\lambda)$ and $\bar{g}_s(\lambda)$ have the same maximizer $\bar{\lambda}^\star$, $\bar{\theta}^\star\in \argmin_{\theta\,\in\,\Theta} \bar{\mathcal{L}}_s(\theta,\bar\lambda^\star
)$, which naturally gives a dual training algorithm in Algorithm~\ref{alg:primal-dual}. 

\begin{algorithm}
	\caption{Constrained Diffusion Models via Dual Training}\label{alg:primal-dual}
	\begin{algorithmic}[1]
		\State \textbf{Input}: total diffusion steps $T$, diffusion parameter $\alpha_t$, total iterations $H$,  
		stepsize $\eta$.
		\State \textbf{Initialize}: $\lambda(1) = 0$.
		\For{$h = 1, \cdots, H$}
		\State Compute model
		$ \displaystyle
		\hat{s}_\theta (h)
		\in  \argmin_{\theta\,\in\,\Theta}
		{\mathcal{L}}_s (\hat{s}_\theta,\lambda(h)).
		$
		\State Update the dual variable
		\[\,
		\lambda_i (h + 1) 
		\;\, = \;\,
		\left[ \lambda_i(h) 
		\, + \,
		\eta \left( \mathbb{E}_{x_0 \,\sim\, q^i, t, x_t}
		\left[\, 
		\norm{\hat{s}_\theta(h)(x_t, t) - \nabla\log q^i(x_t)}^2
		\,\right]  
		\,-\,
		\Tilde{b}^i \right) \right]_+ \text{ for all } i.
		\,\]
		\EndFor
	\end{algorithmic}
\end{algorithm}

Denote $\bar{s}^\star \DefinedAs \hat{s}_{\bar{\theta}^
	\star}$. Thus, $\bar{s}^\star$-induced diffusion model is given by $\bar{p}^\star(\bar{\lambda}^\star)$. Algorithm~\ref{alg:primal-dual} works as a dual ascent method with two natural steps: (i) find a diffusion model with fixed dual variable $\lambda(h)$; and (ii) update the dual variable using the (sub)gradient of the Lagrangian ${\mathcal{L}}_s(\hat{s}_\theta(h),\lambda)$. It is known that Algorithm~\ref{alg:primal-dual} converges to $\bar{\lambda}^\star$ since the dual function $\bar{g}_s(\lambda)$ is concave. However, such convergence in the dual domain doesn't provide optimality guarantee on the primal solution $\bar{p}^\star(\bar{\lambda}^\star)$ due to the non-convexity in parameter space. We next exploit the optimization properties of unparametrized diffusion models in Section~\ref{subsec:unparametrized case} to characterize the optimality of the dual training algorithm. 

\subsection{Optimality analysis of dual training algorithm}\label{subsec:optimality analysis}

We analyze the optimality of $\bar{p}^\star(\bar{\lambda}^\star)$ as measured by its distance to $q_{\text{\normalfont mix}}^{\star}$, i.e., $\text{\normalfont TV}(q_{\text{\normalfont mix}}^{\star},\; \bar p^\star(\bar\lambda^\star))$, where we denote the total variation distance between two probability distributions $p$ and $q$ by $\text{TV}(q, p) \DefinedAs \frac{1}{2}\int |p(x) - q(x)| dx$. We first exploit the convergence analysis of diffusion models, and then characterize the additional error induced by parametrization.

Let us begin with the difference between $\bar{p}^\star(\bar{\lambda}^\star)$ and $q_{\text{mix}}^{(\bar\lambda^\star)}$ at $\bar{\lambda}^\star$. Denote a partial minimizer of the Lagrangian by $\bar p^\star(\lambda) \in \argmin_{\theta\,\in\,\Theta} \bar{\mathcal{L}}_s(\theta, \lambda)$ for $\lambda \geq 0$. Noting that $\minimize_{\theta\,\in\,\Theta} \bar{\mathcal{L}}_s(\theta,\lambda)$ is an unconstrained diffusion problem, we are ready to quantify the difference between $\bar p^\star(\lambda)$ and $q_{\text{mix}}^{(\lambda)}$ for any $\lambda\geq0$ using the convergence theory of diffusion models. To do so, we assume the boundedness of samples from a mixed data distribution $q_{\text{mix}}^{(\lambda)}$ for $\lambda\geq 0$, and a small score estimation error.  

\begin{assumption}[Boundedness of data]\label{ass:bounded samples}
	The data samples generated from $q_{\text{\normalfont mix}}^{(\lambda)}$ are bounded, i.e., $\mathbb{P}\left( \Vert x_0 \Vert \leq T^{c}  \,\vert\, x_0 \sim q_{\text{\normalfont mix}}^{(\lambda)}\right) = 1$ for any $\lambda\geq 0$ and some large constant $c>0$.
\end{assumption}
\begin{assumption}[Boundedness of score estimation error]\label{ass: small approx error for score}
	The score estimator $\hat{s}_\theta(x_t, t)$  estimates the data samples from $q_{\text{\normalfont mix}}^{(\lambda)}$ with bounded score matching error $\varepsilon_{\text{\normalfont score}}$, 
	\[
	\mathbb{E}_{ q_{\text{\normalfont mix}}^{(\lambda)}(x_0),\, t,\, x_t}
	\left[ \,
	\norm{\hat{s}_\theta(x_t, t) - \nabla \log q(x_t)}^2
	\,\right]
	\; \leq \; \varepsilon_{\text{\normalfont score}}^2 
	\]
	for any $\lambda \geq 0$, where $\mathbb{E}_{q_{\text{\normalfont mix}}^{(\lambda)}(x_0),\, t,\, x_t}$ is an expectation over the mixed data distribution $q_{\text{\normalfont mix}}^{(\lambda)}(x_0)$, a uniform distribution over $t$ from $2$ to $T$, and a forward process $q(x_t\,\vert\,x_0)$ given the data sample $x_0$.     
\end{assumption}

Since data samples are bounded, Assumption~\ref{ass:bounded samples} is mild in practice.  Assumption~\ref{ass: small approx error for score} is the typical score matching error that is near zero if the function class $\mathcal{S}_\theta$ is sufficiently rich.

With Assumptions~\ref{ass:bounded samples} and~\ref{ass: small approx error for score}, below we bound the TV distance between $q_{\text{mix}}^{(\lambda)}$ and $\bar{p}^\star(\lambda)$ using the convergence theory of diffusion models from~\cite{li2023towards}; see Appendix~\ref{app: bounded TV distance} for proof. 

\begin{lemma}[Convergence of diffusion model]\label{lem: bounded TV distance}
	Let Assumptions~\ref{ass:bounded samples} and~\ref{ass: small approx error for score} hold. 
	Then, the TV distance from $\bar p^\star(\lambda)$ to $q_{\text{\normalfont mix}}^{(\lambda)}$ is bounded by 
	\begin{equation}\label{eq: bounded TV distance main}
		\text{\normalfont TV}\left(q_{\text{\normalfont mix} }^{(\lambda)},\; \bar p^\star(\lambda) \right) 
		\; \leq \; \sqrt{\frac{1}{2} D_{\text{\normalfont KL}}\left(q_{\text{\normalfont mix}}^{(\lambda)} \,\Vert\, \bar p^\star(\lambda)\right)} 
		\; \lesssim \;
		\frac{d^2\, \log^3T}{\sqrt{T}} \,+\, \sqrt{d}\, \left( \log^2T \right) \varepsilon_{\text{\normalfont score}}.
	\end{equation}
\end{lemma}

Lemma~\ref{lem: bounded TV distance} states that the TV distance between $q_{\text{mix}}^{(\lambda)}$ and $\bar{p}^\star(\lambda)$ decays to zero with a sublinear rate $O(\frac{1}{\sqrt{T}})$, up to a score matching error $O(\varepsilon_{\text{\normalfont score}})$. When the diffusion time $T$ is large, the TV distance between $q_{\text{mix}}^{(\lambda)}$ and $\bar{p}^\star(\lambda)$ is dominated by the score matching error. Substitution of $\lambda = \bar{\lambda}^\star$ into~\eqref{eq: bounded TV distance main} yields an upper bound on $ \text{\normalfont TV}(q_{\text{\normalfont mix}}^{(\bar\lambda^\star)},\; \bar p^\star(\bar\lambda^\star))$, which is the second term of the inequality
\begin{equation}\label{eq:TV distance triangle inequality}
	\begin{array}{rcl}
		\text{\normalfont TV}\left(q_{\text{\normalfont mix}}^{\star},\; \bar p^\star(\bar\lambda^\star)\right)
		& \leq & 
		\displaystyle
		\text{\normalfont TV}\left(q_{\text{\normalfont mix}}^{\star},\; q_{\text{\normalfont mix}}^{(\bar\lambda^\star)}\right)
		\, + \,
		\text{\normalfont TV}\left(q_{\text{\normalfont mix}}^{(\bar\lambda^\star)},\; \bar p^\star(\bar\lambda^\star)\right).
	\end{array}
\end{equation}
Next, we quantify the gap between $\bar{\lambda}^\star$ and $\lambda^\star$, which lets us bound the first term on the RHS of~\eqref{eq:TV distance triangle inequality} and complete the optimality analysis. To analyze the parametrized optimal dual variable $\bar{\lambda}^\star$, we introduce the richness of the parametrized class $\mathcal{S}_\theta$ and redundancy of constraints at $\hat{s}^\star$ below.

\begin{assumption}[Richness of parametrization]\label{ass:richness}
	For any function $\hat{s} \in \mathcal{S}$, there exists parameter $\theta \in \Theta$ such that $\Vert \hat{s}_\theta - \hat{s}\Vert_{L_2} \leq \nu$, where $\norm{\cdot}_{L_2}$ is with respect to the forward process.
\end{assumption}

\begin{assumption}[Redundancy of constraints]\label{ass:redundancy}
	There exists $\sigma >0$ such that
	\begin{equation}
		\inf_{\Vert\lambda\Vert \,=\, 1}\; \left\Vert\,
		\sum_{i \,=\, 1}^m 
		\lambda_i\, 
		\nabla_{\hat{s}} \, \mathbb{E}_{q^i(x_0),\, t,\, x_t}
		\left[\,
		\hat{s}^\star(x_t, t) - \nabla \log q(x_t)
		\,\right] 
		\,\right\Vert_{L_2}
		\;\geq \;
		\sigma
	\end{equation}
	where $\nabla_{\hat{s}}$ is the Fréchet derivative over the function $\hat{s}$ and $\hat{s}^\star$ is a solution to Problem~\eqref{eq:KL-constrained likelihood maximization unparametrization Elbo MSE}.
\end{assumption}

Assumption~\ref{ass:richness} is mild since the gap is small for expressive neural networks~\cite{mei2023deep,han2024neural}. Assumption~\ref{ass:redundancy} captures the linear independence of constraints, which is similarly used in optimization~\cite{bertsekas2016nonlinear}.

Due to Assumption~\ref{ass:bounded samples}, we set the function class $\mathcal{S}$ to be bounded $\norm{\,\hat{s}\,}_{L_2} \leq T^c \DefinedAs R$. Using Problem~\eqref{eq:KL-constrained likelihood maximization unparametrization Elbo MSE}, we prove that the unparametrized dual function $g_s(\lambda)$ is differentiable, and strongly-concave over $\mathcal{H}$ with parameter $\mu$, where $\mathcal{H}\DefinedAs \{\gamma \lambda^\star + (1-\gamma) \bar\lambda^\star, \gamma \in [0, 1] \}$ and $\mu \DefinedAs \left({\sigma}/\left({1+\max\left(\norm{\lambda^\star}_1, \norm{\bar\lambda^\star}_1\right)}\right)\right)^2$,which leads to Lemma~\ref{lem:dual un/parametrized gap}; see Appendix~\ref{app:dual un/parametrized gap} for their proofs.

\begin{lemma}\label{lem:dual un/parametrized gap}
	Let Assumptions~\ref{ass:richness} and~\ref{ass:redundancy} hold. Then, $\norm{\bar\lambda^\star - \lambda^\star}^2 \leq \frac{8}{\mu} R\left(1+\norm{\bar\lambda^\star}_1\right) \nu$.
\end{lemma}

Since $\text{\normalfont TV}\left(q_{\text{\normalfont mix}}^{\star},\; q_{\text{\normalfont mix}}^{(\bar\lambda^\star)}\right)$ is bounded by $\norm{\bar{\lambda}^\star-\lambda^\star}_1$ (see Appendix~\ref{app:dual un/parametrized gap}), application of Lemma~\ref{lem: bounded TV distance} and Lemma~\ref{lem:dual un/parametrized gap} to~\eqref{eq:TV distance triangle inequality} leads to Theorem~\ref{thm: bounded TV distance ideal}; see Appendix~\ref{app: bounded TV distance ideal} for proof.

\begin{theorem}[Optimality of constrained diffusion model]\label{thm: bounded TV distance ideal}
	Let Assumptions~\ref{ass:feasibility_KL}--\ref{ass:redundancy} hold. Then, the total variation distance between $\bar p^\star(\bar\lambda^\star)$ and $q_{\text{\normalfont mix}}^{\star}$ is upper bounded by
	\[
	\text{\normalfont TV}\left(q_{\text{\normalfont mix}}^{\star},\; \bar p^\star(\bar\lambda^\star)\right)  
	\; \lesssim \;
	\frac{d^2 \,\log^3T}{\sqrt{T}} 
	\,+\,
	\sqrt{\frac{8}{\mu} \,m\,  R \left(1+\norm{\bar\lambda^\star}_1\right) \nu}
	\, + \,
	\sqrt{d}\, \left( \log^2T  \right)\, \varepsilon_{\text{\normalfont score}}.
	\]
\end{theorem}

Theorem~\ref{thm: bounded TV distance ideal} states that the TV distance between $\bar{p}^\star(\bar{\lambda^\star})$ and $q_{\text{mix}}^\star$ decays to zero with a sublinear rate $O(\frac{1}{\sqrt{T}})$, up to a parametrization gap $O(\sqrt{\nu})$ and a prediction error $O(\varepsilon_{\text{\normalfont score}})$. When the parametrization is rich enough, the parametrization gap $\nu$ and the prediction error $O(\varepsilon_{\text{\normalfont score}})$ are nearly zero. In this case, if the diffusion time $T$ is large, then  $\bar{p}^\star(\bar{\lambda^\star})$ is close to $q_{\text{mix}}^{\star}$ in TV distance, which recovers the ideal optimal constrained model in the unparametrized case in Section~\ref{subsec:unparametrized case}.

\subsection{Practical dual training algorithm}\label{subsec:primal-dual training}

Having established the optimality of our dual training method, we futher turn Algorithm~\ref{alg:primal-dual} into a practical algorithm. First, we relax the computation of a diffusion model $\hat{s}_{\theta} (h)$ in line~4 of Algorithm~\ref{alg:primal-dual} to be approximate: ${\mathcal{L}}_s(\hat{s}_{\theta} (h), \lambda(h))
\leq   
\displaystyle
\min_{\theta\,\in\,\Theta}
{\mathcal{L}}_s(\hat{s}_{\theta}, \lambda(h))
+
\varepsilon_{\text{approx}}^2$,
where $\varepsilon_{\text{approx}}^2$ is an approximation error of training a diffusion model given $\lambda(h)$. Second, we replace the gradient in line 5 of Algorithm~\ref{alg:primal-dual} by a stochastic gradient $\hat{\mathbb{E}}_{x_0 \,\sim\, q^i,\, t,\, x_t}
\big[\, 
\norm{\hat{s}_\theta(x_t, t) - \nabla \log q(x_t)}^2
\,\big]$, which enables Algorithm~\ref{alg:primal-dual} to be a stochastic algorithm, where  $\hat{\mathbb{E}}_{x_0 \,\sim\, q^i,\, t,\, x_t}$ is an unbiased estimate of $ 
\mathbb{E}_{x_0 \, \sim\, q^i,\, t,\, x_t}$. To analyze this approximate and stochastic variant of Algorithm~\ref{alg:primal-dual}, it is useful to introduce the maximum parametrized dual function in history up to step $h$ by $\bar{g}_{\text{best}}(h) \DefinedAs \max_{h'\,\leq\, h} \bar{g}_s(\lambda(h'))$, and an upper bound of the second-order moment of stochastic gradient  
$
S^2
\DefinedAs 
\sum_{i\,=\,1}^m \mathbb{E} \big[\,\big(\,
\hat{\mathbb{E}}_{x_0\, \sim\, q^i,\, t,\, x_t}
\big[\, 
\norm{\hat{s}_\theta(h)(x_t, t) - \nabla q(x_t)}^2
\,\big]  - \Tilde{b}^i
\,\big)^2\,\vert\, \lambda(h)
\,\big].
$  

Denote the dual variable that achieves $\bar{g}_{\text{best}} (h)$ by  $\lambda_{\text{\normalfont best}}$. To bound the TV distance between $\bar p^\star(\bar\lambda_{\text{\normalfont best}})$ and $q_{\text{\normalfont mix}}^{\star}$, we check the TV distance between $q_{\text{\normalfont mix}}^{(\bar\lambda_{\text{\normalfont best}})}$ and  $\bar p^\star(\bar\lambda_{\text{\normalfont best}})$ using Lemma~\ref{lem: bounded TV distance}. The rest is to analyze the convergence of $\bar\lambda_{\text{\normalfont best}}$ to $\lambda^\star$ via application of martingale convergence. We defer their proofs to Appendix~\ref{app:algorithm output vs unparametrized solution} and present the optimality of $\bar p^\star(\bar\lambda_{\text{\normalfont best}})$ in Theorem~\ref{thm:algorithm output vs unparametrized solution}.

\begin{theorem}[Optimality of approximate constrained diffusion model]\label{thm:algorithm output vs unparametrized solution}
	Let Assumptions~\ref{ass:feasibility_KL}--\ref{ass:redundancy} hold. Then, the total variation distance between $\bar p^\star(\bar\lambda_{\text{\normalfont best}})$ and $q_{\text{\normalfont mix}}^{\star}$ is upper bounded by
	\[
	\text{\normalfont TV}\left(q_{\text{\normalfont mix}}^{\star},\; \bar p^\star(\bar\lambda_{\text{\normalfont best}})\right)  
	\; \lesssim \;
	\frac{d^2 \,\log^3T}{\sqrt{T}} 
	\, + \, 
	\frac{8 R \left(1+\norm{\bar\lambda_{\text{\normalfont best}}}_1\right)}{\mu}\nu
	\, + \,
	\sqrt{d} \,\left(\log^2T\right) \varepsilon_{\text{\normalfont score}} 
	\, + \, 
	\frac{2}{\mu}\,\varepsilon_{\text{\normalfont approx}}^2
	\, + \,
	\frac{\eta \, S^2}{\mu}.
	\]
\end{theorem}

Theorem~\ref{thm:algorithm output vs unparametrized solution} states that the TV distance between $\bar{p}^\star(\bar{\lambda}_{\text{best}})$ and $q_{\text{mix}}^\star$ decays to zero with a sublinear rate $O(\frac{1}{\sqrt{T}})$, up to a parametrization gap $O(\nu)$, a score matching error $O(\varepsilon_{\text{\normalfont score}})$, an approximation error $O(\varepsilon_{\text{\normalfont approx}})$, and stepsize $O(\eta)$. When the parametrization is rich enough, the parametrization gap $\nu$ and the score matching error $O(\varepsilon_{\text{\normalfont score}})$ are near zero. Thus, if the diffusion time $T$ is large, then the closeness of  $\bar{p}^\star(\bar{\lambda}_{\text{best}})$ to $q_{\text{mix}}^{\star}$ in TV distance is governed by the appproximation error and stepsize. 

\section{Computational experiments}\label{sec:experiments}

We demonstrate the effectiveness of constrained diffusion models trained by our dual training algorithm in two constrained settings in Section~\ref{subsec:fair image generation}; see Appendix~\ref{app:implementation} for experimental details.

\noindent\textbf{Fairness to underrepresented classes.} We train constrained diffusion models over three datasets: MNIST digits~\cite{yann2020the}, Celeb-A faces~\cite{liu2018large}, and Image-Net\footnote[1]{We use a subset of ten classes from Image-Net: `Tench Fish', `English Springer Dog', `Cassette Player', `Chain saw', `Church', `French Horn', `Garbage Truck', `Gas Pump', `Golf Ball', `Parachute.'}~\cite{ILSVRC15}. For MNIST and Image-Net, we create a dataset for the distribution $q$ in~\eqref{eq:KL-constrained likelihood maximization parametrization Elbo MSE} by taking a subset of the dataset with equal number of samples from each class. Then we make some classes under-represented by removing their samples. For each distribution $q^i$, we use samples from the associated underrepresented class. For Celeb-A, our approach is similar to MNIST except we don't remove any samples due to the existence of class imbalance in the dataset (58\% female vs 42\% male). For Image-Net, since the images are of high resolution, we employ the latent diffusion scheme~\cite{rombach2022highresolution} by imposing distribution constraints in latent space. Figures~\ref{fig:fairness_MNIST}--\ref{fig:fairness_Image-Net} show that our constrained model samples more often from the underrepresented classes (MNIST: 4, 5, 7; Celeb-A: male; Image-Net: `Cassette player', `French horn', and `Golf ball'), leading to a more uniform sampling over all classes. This reflects our theoretical insights on promoting fairness for minority classes (see Section~\ref{subsec:fair image generation}). Quantitatively, we observe \emph{fairly lower FID scores} when training over \emph{the same dataset but with constraints} (see Appendix~\ref{app:implementation} for further discussion on FID scores). Furthermore, our Image-Net experiment shows that our approach extends to the state-of-the-art diffusion models in latent space.

\begin{figure}[h!]
	\centering
	\begin{tabular}{rccc}
		{\rotatebox{90}{\;\;\;\;
				\;\;
				frequencies}}
		&
		
		\tabfigure{width=0.38\textwidth}{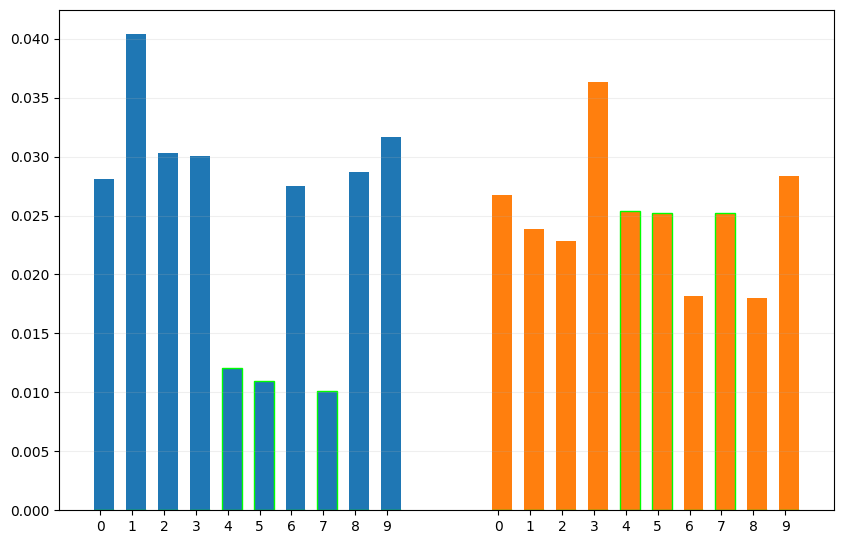}
		\;\;
		& \;\;
		\includegraphics[width=0.22\linewidth]{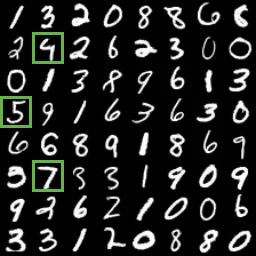}
		&
		\includegraphics[width=0.22\linewidth]{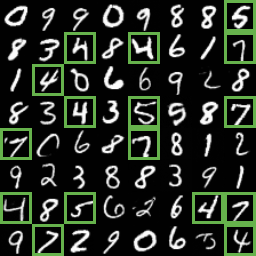}
		\\
		&
		classes
		& 
		& 
	\end{tabular}
	\caption{Generation performance comparison of constrained and unconstrained models that are trained on MNIST with three minorities: 4, 5, 7. (\,Left\,) Frequencies of ten digits that are generated by an unconstrained model (\crule[myblue]{0.5cm}{0.15cm}) and our constrained model (\crule[myorange]{0.5cm}{0.15cm}); (\,Middle\,) Generated digits from unconstrained model (\,FID 15.9\,); (\,Right\,) Generated digits from our constrained model (\,FID 13.4\,).}
	\label{fig:fairness_MNIST}
\end{figure}

\begin{figure}[h!]
	\centering
	\begin{tabular}{rccc}
		{\rotatebox{90}{\;\;\;\;
				\;\;
				frequencies}}
		&
		
		\tabfigure{width=0.38\textwidth}{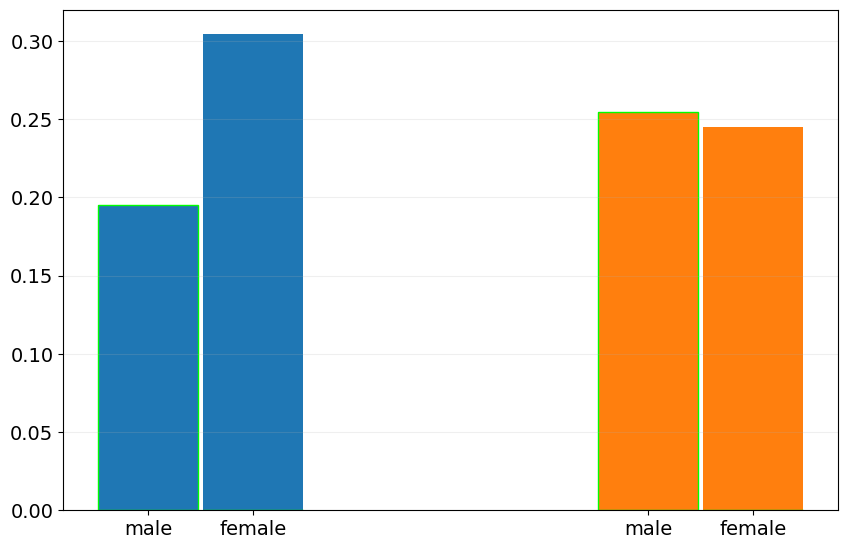}
		\;\;
		& \;\;
		\includegraphics[width=0.22\linewidth]{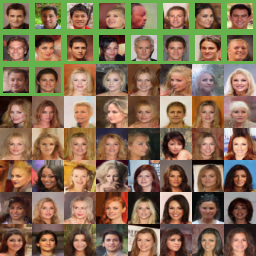}
		&
		\includegraphics[width=0.22\linewidth]{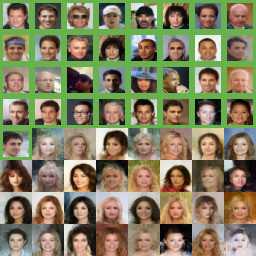}
		\\
		&
		classes
		& 
		
		& 
	\end{tabular}
	\caption{Generation performance comparison of constrained and unconstrained models that are trained on Celeb-A with male minority. (\,Left\,) Frequencies of two genders that are generated by an unconstrained model (\crule[myblue]{0.5cm}{0.22cm}) and our constrained model (\crule[myorange]{0.5cm}{0.22cm}); (\,Middle\,) Generated faces from unconstrained model (\,FID 19.6\,); (\,Right\,) Generated faces from our constrained model (\,FID 11.6\,).}
	\label{fig:fairness_Celeb}
\end{figure}

\begin{figure}[h!]
	\centering
	\begin{tabular}{rccc}
		{\rotatebox{90}{\;\;\;\;
				\;\;
				frequencies}}
		&
		
		\tabfigure{width=0.38\textwidth}{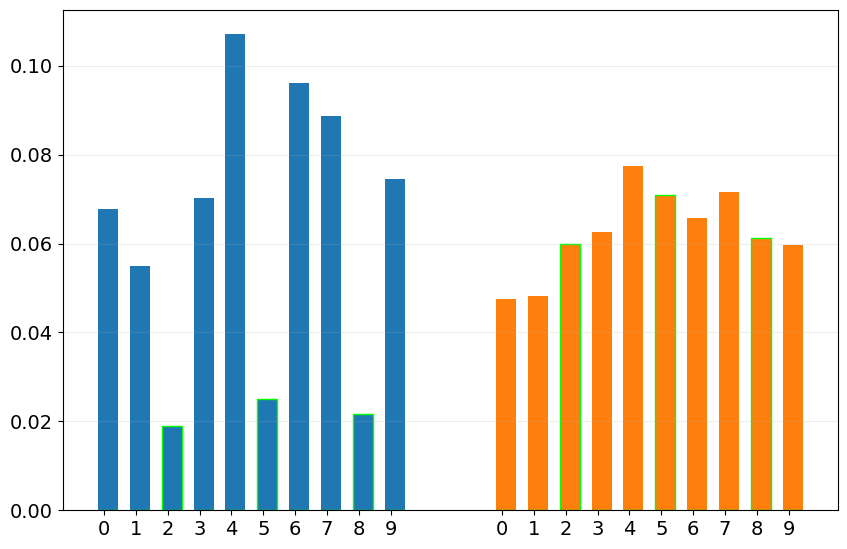}
		\;\;
		& \;\;
		\includegraphics[width=0.22\linewidth]{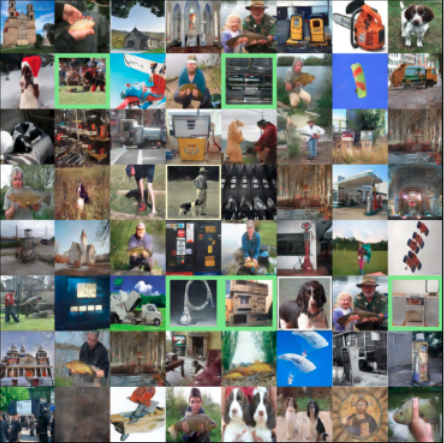}
		&
		\includegraphics[width=0.22\linewidth]{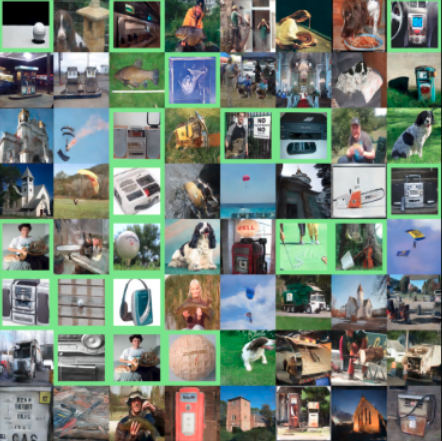}
		\\
		&
		classes
		& 
		
		& 
	\end{tabular}
	\caption{Generation performance comparison of constrained and unconstrained models that are trained on Image-Net with minority classes: `Cassette player' (2), `French horn' (5), and `Golf ball' (8). (\,Left\,) Frequencies of ten classes that are generated by an unconstrained model (\crule[myblue]{0.5cm}{0.15cm}) and our constrained model (\crule[myorange]{0.5cm}{0.15cm}); (\,Middle\,) Generated images from unconstrained model (\,FID 36.0\,); (Right) Generated images from our constrained model (\,FID 27.3\,).}
	\label{fig:fairness_Image-Net}
\end{figure}

\begin{figure}[h!]
	\centering
	\begin{tabular}{rccc}
		{\rotatebox{90}{\;\;\;\;
				\;\;
				frequencies}}
		&
		
		\tabfigure{width=0.38\textwidth}{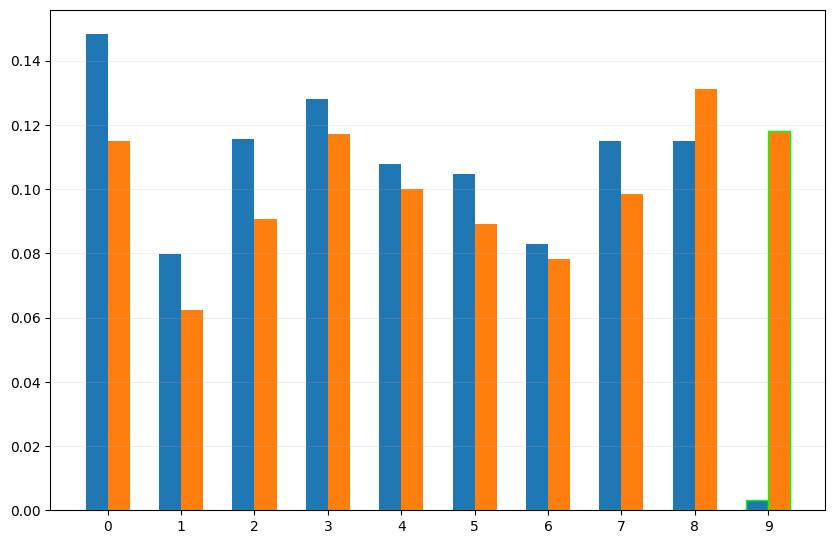}

		\;\; & \;\;
		\includegraphics[width=0.22\linewidth]{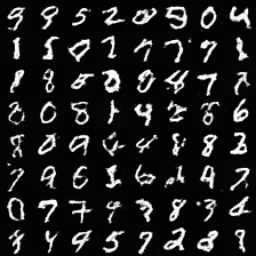}
		&
		\includegraphics[width=0.22\linewidth]{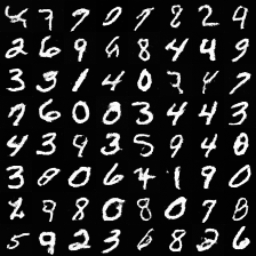}
		\\
		&
		classes
		& 
		
		& 
	\end{tabular}
	\caption{Fine-tuning performance comparison of constrained and unconstrained models that are trained on MNIST. (\,Left\,) Frequencies of ten digits that are generated by a pre-trained model without digit 9 (\crule[myblue]{0.5cm}{0.15cm}) and our fine-tuned constrained model (\crule[myorange]{0.5cm}{0.15cm}); (\,Middle\,) Generated digits from unconstrained model (\,FID 45.9\,); (\,Right\,) Generated digits from our constrained model (\,FID 25.2\,).}
	\label{fig:fine-tuning}
\end{figure}

\noindent\textbf{Adapting pretrained model to new data.} Given a pretrained diffusion model over some original dataset $\mathcal{D}_{\text{pretrain}}$, we fine-tune the pretrained model for generating data that resemble $\mathcal{D}_{\text{new}}$. To cast this problem into~\eqref{eq:KL-constrained likelihood maximization parametrization}, we let the data distribution be $\mathcal{D}_{\text{new}}$, i.e.,  $q(x_{0:T}) = q_{\text{new}}(x_{0:T})$ and the constrained distribution be the pre-trained model, i.e., $q^i(x_{0:T}) = p_{\theta_{\text{pre}}}(x_{0:T})$. In our experiments, we pretrain a diffusion model on a subset of MNIST digits excluding a class of digits (MNIST: 9), and fine-tune this model using samples of the excluded digit. Figure~\ref{fig:fine-tuning} shows that our constrained fine-tuned model samples from the new class  as well as all previous classes, whereas the model fine-tuned without the constraint quickly overfits to the new dataset (see Appendix~\ref{app:implementation} for details). 
Our constrained model generates much better high-quality samples than the unconstrained model.

\vspace{-1.1ex}
\section{Conclusion}\label{sec:conclusion}

We have presented a constrained optimization framework for training diffusion models under distribution constraints. We have developed a Lagrangian-based dual algorithm to train such constrained diffusion models. Our theoretical analysis shows that our constrained diffusion model generates new data from an optimal mixture data distribution that satisfies the constraints, and we have demonstrated the effectiveness of our distribution constraints in reducing bias across three widely-used datasets. 

This work directly stimulates several research directions: (i) extend our distribution constraints to other domain constraints, e.g., mirror
diffusion~\cite{liu2024mirror}; (ii) incorporate conditional generations, e.g., text-to-image generation~\cite{friedrich2023fair,shen2023finetuning}, into our constrained diffusion models; (iii) conduct experiments with text-to-image datasets to identify and address biases; (iv) improve the convergence theory using more advanced diffusion processes. 



\newpage
\section*{Acknowledgments}

We thank reviewers and program chairs for providing helpful comments.

\bibliography{main.bib}
\bibliographystyle{abbrv} %

\newpage
\appendix
\clearpage

~\\
\centerline{{\fontsize{14}{14}\selectfont \textbf{Supplementary Materials for }}}

\vspace{6pt}
\centerline{\fontsize{13.5}{13.5}\selectfont \textbf{
	``Constrained Diffusion Models via Dual Training''}}

\vspace{6pt}
\section{Details on ELBO}\label{app:ELBO}

Recall the evidence lower bound (ELBO),$$E(p; q) \; \DefinedAs\; \mathbb{E}_{q(x_0)} \mathbb{E}_{q(x_{1:T}\,\vert\,x_0)} \log \frac{p(x_{0:T})}{q(x_{1:T}\,\vert\,x_0)},$$ we can utilize conditionals to expand it into
\[
\begin{array}{rcl}
	E(p; q ) 
	& = & \displaystyle
	\underbrace{\mathbb{E}_{q(x_0)} \mathbb{E}_{q(x_1\,\vert\,x_0)} \left[\,\log p(x_0\,\vert\,x_1)\,\right]}_{\text{\normalsize reconstruction likelihood}}
	\, - \, 
	\underbrace{\mathbb{E}_{q(x_0)} \left[\, D_{\text{KL}}(q(x_T\,\vert\,x_0)\,\Vert\, p(x_T))\,\right]}_{\text{\normalsize final latent mismatch}}
	\\[0.4cm]
	&& \displaystyle - \,\underbrace{\sum_{t\,=\,2}^T \mathbb{E}_{q(x_0)} \mathbb{E}_{q(x_t\,\vert\,x_0)} \left[\, D_{\text{KL}}\left(q(x_{t-1}\,\vert\,x_t,x_0)\,\Vert\, p (x_{t-1}\,\vert\,x_t)\right)\,\right] }_{\text{\normalsize denoising matching term}}
\end{array}
\]
where the first term is the reconstruction likelihood of the original data given the first latent $x_1$, the second term is the mismatch between the final latent distribution and the Guassian prior, and the last summation measures the mismatch between the denoising transitions from forward/backward processes. With the variance schedule described in Section~\ref{subsec:forward backward}, it is known that the reconstruction likelihood and final latent mismatch  are negligible, and thus the approximation in~\eqref{eq:ELBO expansion} is almost exact, which is our focal setting of this paper.

We next focus on one summand of the denoising matching term,
\[
\mathbb{E}_{q(x_0)} \mathbb{E}_{q(x_t\,\vert\,x_0)} 
\left[\,
D_{\text{KL}}\left(q(x_{t-1}\,\vert\,x_t,x_0)\,\Vert\, p (x_{t-1}\,\vert\,x_t)\right)
\,\right].
\]
Application of the reparametrization trick leads to $x_t = \sqrt{\alpha_t} x_{t-1} + \sqrt{1-\alpha_t} \epsilon_{t-1}$, where $\epsilon_{t-1} \sim \mathcal{N}(0, I)$ is a white noise sample. Using Bayes rule, we can express $q(x_{t-1}\,\vert\,x_t, x_0)$ as a Guassian distribution 
\[
\mathcal{N}(x_{t-1}; \mu_q(x_t), \sigma_q(t)I)
\]
where $\mu_q(x_t) = \frac{1}{\sqrt{\alpha_t}}x_t+\frac{1-\alpha_t}{\sqrt{\alpha_t}}\nabla \log q(x_t)$ is the mean and $\sigma_q^2(t) = \frac{(1-\alpha_t)(1-\bar\alpha_{t-1})}{1-\bar\alpha_t}$ is the variance.

To stay close to the ground-truth backward conditional $q(x_{t-1}\,\vert\,x_t,x_0)$ as much as possible, we take $p(x_{t-1}\,\vert\,x_t)$ to be the same as $q(x_{t-1}\,\vert\,x_t,x_0)$ except replacing $\nabla\log p(x_t)$ by $\hat{s}(x_t,t)$ and $\sigma_q^2(t)$ by $\sigma_p^2(t)$,
\[
\mathcal{N}(x_{t-1}; \hat\mu(x_t), \sigma_p(t)I)
\]
where $\hat\mu(x_t) = \frac{1}{\sqrt{\alpha_t}}x_t+\frac{1-\alpha_t}{\sqrt{\alpha_t}}\hat{s}(x_t, t)$. Thus,
\[
\begin{array}{rcl}
	&& \!\!\!\!  \!\!\!\!  \!\!
	\displaystyle
	D_{\text{KL}}\left(q(x_{t-1}\,\vert\,x_t,x_0)\,\Vert\, p (x_{t-1}\,\vert\,x_t)\right) 
	\\[0.2cm]
	&  = & D_{\text{KL}}(\mathcal{N}(x_{t-1}; \mu_q(x_t), \sigma_q(t)I)\,\Vert\, \mathcal{N}(x_{t-1}; \hat\mu(x_t), \sigma_p(t)I))
	\\[0.2cm]
	& = & \displaystyle
	\frac{1}{2} 
	\left(
	d\log \frac{\sigma_p^2(t)}{\sigma_q^2(t)}
	-
	d
	+
	d \frac{\sigma_p^2(t)}{\sigma_q^2(t)}
	+ 
	\frac{1}{\sigma_q^2(t)}
	\norm{\mu_q(x_t,x_0) - \hat\mu(x_t,x_0)}^2
	\right)
	\\[0.2cm]
	& = & \displaystyle\underbrace{
		\frac{1}{2} 
		\left(
		d\log \frac{\sigma_p^2(t)}{\sigma_q^2(t)}
		-
		d
		+
		d \frac{\sigma_p^2(t)}{\sigma_q^2(t)}
		\right)
	}_{\text{\normalsize variance mismatch}}
	\, + \,
	\underbrace{\frac{1}{2\sigma_q^2(t)}
		\frac{ (1-\alpha_t)^2}{\alpha_t}
		\norm{\hat{s}(x_t,t) - \nabla\log q(x_t)}^2}_{\text{ \normalsize prediction loss}}
\end{array}
\]
where the second equality is due to the KL Divergence between two Gaussians. Since $\sigma_p^2(t)$ and $\sigma_q^2(t)$ are constants, the variance mismatch term is irrelevant to optimization. Denote 
$\omega_t 
\DefinedAs 
\frac{(1-\alpha_t)^2}{2\sigma_q^2(t)\alpha_t}$ and $\bar{\omega} \DefinedAs \sum_{t\,=\,2}^T \omega_t$. We can define a discrete distribution over the set $\{2, \ldots, T\}$ as $p_\omega(t) \DefinedAs \frac{\omega_t}{\bar\omega}$. Also denote $v \DefinedAs \sum_{t\,=\,2}^T \frac{1}{2} 
\left(
d\log \frac{\sigma_p^2(t)}{\sigma_q^2(t)}
-
d
+
d \frac{\sigma_p^2(t)}{\sigma_q^2(t)}
\right)$.
Hence, the ELBO maximization: 
$\maximize_{p} E(p; q)$, is equivalent to the quadratic loss minimization,
\[
\minimize_{\hat{s}} \;\; 
v 
\,+\, 
\bar{\omega}\,
\mathbb{E}_{x_0, \,t, \,x_t}
\left[\,  
\norm{\hat{s}(x_t, t) - \nabla\log q(x_t)}^2
\,\right]
\]
where $\mathbb{E}_{x_0,\, t,\, x_t}$ is an expectation over the data distribution $q(x_0)$, the discrete distribution $p_\omega(t)$  from $2$ to $T$, and a forward process $q(x_t\,\vert\,x_0)$ given the data sample $x_0$. Since shifting an objective function by a constant and scaling an objective function by a constant  don't change the solution of an optimization problem, we omit constants $v$ and $\bar{\omega}$ for brevity, and only emphasize them whenever it is necessary. Hence, the ELBO maximization equals the quadratic loss minimization, $$\minimize_{\hat{s}} \;\;   \mathbb{E}_{x_0,\, t,\, x_t}
\left[\,  
\norm{\hat{s}(x_t, t) - \nabla \log q(x_t)}^2
\,\right]
$$ up to some scaling and shifting constants,
where $\mathbb{E}_{x_0, t, x_t}$ is an expectation over the data distribution $q(x_0)$, the discrete distribution $p_\omega(t)$  from $2$ to $T$, and a forward process $q(x_t\,\vert\,x_0)$ given the data sample $x_0$.  
In practice, however, we have to parametrize the estimator $\hat{s}(x_t,t)$ as $\hat{s}_\theta(x_t,t)$ with parameter $\theta\in\Theta$, 
\[
\minimize_{\theta\,\in\,\Theta} \;\;  
\mathbb{E}_{x_0,\, t,\, x_t}
\left[ \,
\norm{\hat{s}_\theta(x_t, t) - \nabla \log q(x_t)}^2
\, \right] 
\]
which is our focal objective of generative modeling. A parametrized representation of $p(x_{t-1}\,\vert\,x_t)$ associated with $\hat{s}_\theta(x_t, t)$ is denoted by $p_\theta(x_{t-1}\,\vert\,x_t)$ and the backward process has a parametrized joint distribution $p_\theta(x_{0:T})$. We remark that the above prediction problem can be reformulated as data or noise predictions~\cite{luo2022understanding}, with our results directly transferrable to these formulations.

\section{Proofs}\label{app:proofs}

We provide proofs of all lemmas and theorems in the main paper.

\subsection{Proof of Lemma~\ref{lem:equivalence}}\label{app:equivalence}

\begin{proof}
	The ELBO maximization has the same optimal solution with the KL divergence minimization because of the equality~\eqref{eq:KL forward/backward}. This directly proves the second equivalence. Next, we relate these two problems to the likelihood maximization problem.
	
	We note that the KL divergence is non-negative and is zero if and only if two distributions are the same. Since $q(x_{0:T}) \in \mathcal{P}$ for large $T$, the solution of the ELBO maximization and KL divergence minimization is given by $p^\star = q$. For the KL divergence minimization, the optimal value is zero. For the optimal value of the ELBO maximization, from~\eqref{eq:KL forward/backward} it follows that:
	\begin{equation}
		E(p^\star; q) 
		\; = \;
		\mathbb{E}_{q(x_0)}[\, \log q(x_0) \,] - D_{\text{KL}}(q(x_{0:T})\,\Vert\,  p^\star(x_{0:T})) 
		\; = \;
		\mathbb{E}_{q(x_0)}[ \, \log q(x_0) \,].
	\end{equation}
	
	It is clear that the likelihood maximization problem $\maximize_p\mathbb{E}_{q(x_0)}[ \,\log p(x_0)\, ]$ is equivalent to $\minimize_p D_{\text{KL}}(q(x_0)\,\Vert\, p(x_0))$. Therefore, any distribution $p^\star (x_{0:T})$  whose marginal satisfies $p^\star (x_0) = q(x_0)$, will be a solution of the likelihood maximization problem. This includes the solution of the KL  divergence minimization and ELBO maximization which is $p^\star  =  q$. Therefore,
	\begin{equation}
		\maximize_{p \,\in\, \mathcal{P}}\; E(p; q)
		\;\Rightarrow\; 
		\maximize_{p \,\in\, \mathcal{P}}\;
		\mathbb{E}_{q(x_0)}[\, \log p(x_0) \,]
	\end{equation}
	which concludes the proof.
	
\end{proof}

\subsection{Proof of Lemma~\ref{lem:duality}}\label{app:duality}

\begin{proof}
	It is straightforward to check the zero duality gap in convex optimization; see e.g.,~\cite[Proposition 5.3.2]{bertsekas2016nonlinear}. Furthermore, for a convex optimization problem, an optimal dual variable $\lambda^\star$ that maximizes the dual function is a geometric multiplier. Hence, $(p^\star, \lambda^\star)$ is an optimal primal-dual pair of the convex optimization problem.
\end{proof}

\subsection{Proof of Theorem~\ref{thm:solution of unparameterized constrained optimization}}\label{app:solution of unparameterized constrained optimization}
\begin{proof}
	From the strong duality in Lemma~\ref{lem:duality}, it is known from~\cite[Proposition 5.1.4]{bertsekas2016nonlinear} that $\lambda^\star$ is also a geometric multiplier. Thus,  Problem~\eqref{eq:KL-constrained unparameterized optimization} reduces to an unconstrained problem,
	\begin{equation}\label{eq:unconstrained plug-in}
		\minimize_{p\,\in\,\mathcal{P}}\;\; \mathcal{L}(p,\lambda^\star)
	\end{equation}
	where the objective function results from plugging an optimal dual variable $\lambda^\star$ into  Lagrangian $\mathcal{L}(p,\lambda)$.    
	
	By the definition of Lagrangian, 
	\[
	\begin{array}{rcl}
		\mathcal{L}(p, \lambda) 
		& = & \displaystyle
		D_{\text{KL}}(q(x_{0:T})\,\Vert\,  p(x_{0:T})) + \sum_{i\,=\,1}^m \lambda_i  \left(D_{\text{KL}}\left(q^i(x_{0:T})\,\Vert\,  p(x_{0:T})\right) - b_i\right)
		\\[0.2cm]
		& = & \displaystyle
		-\, E(p; q) - \sum_{i\,=\,1}^m \lambda_i\, E(p; q^i)
		\\[0.2cm]
		&  & \displaystyle +\, 
		\mathbb{E}_{q(x_0)}\left[ \log q(x_0) \right] + \sum_{i\,=\,1}^m \lambda_i \left( \mathbb{E}_{q(x_0)}\left[ \log q^i(x_0) \right] -b_i \right)
	\end{array}
	\]
	By taking $\lambda = \lambda^\star$, Problem~\eqref{eq:unconstrained plug-in} is equivalent to
	\begin{equation}\label{eq:lagrangian elbo}
		\displaystyle
		\maximize_{p\, \in \, \mathcal{P}}\;\; E(p; q)  
		\, + \,  
		\sum_{i\,=\,1}^m \lambda_{i}^\star \, E(p; q^i).
	\end{equation}
	From the definition of ELBO, we know that
	\[
	E(p; q)+\sum_{i\,=\,1}^m \lambda_{i}^\star E(p; q^i)
	\; = \; 
	\left(\mathbb{E}_{q(x_0)}  + \sum_{i\,=\,1}^m \lambda_{i}^\star \mathbb{E}_{q^i(x_0)} \right)\mathbb{E}_{q(x_{1:T}\,\vert\,x_0)} \log \frac{p(x_{0:T})}{q(x_{1:T}\,\vert\,x_0)}
	\]
	where we use the fact that the forward processes have the same marginal distribution given any initial data samples. Normalization of initial data distributions leads to $q_{\text{mix}}^{(\lambda^\star)}$.  
	Thus, Problem~\eqref{eq:lagrangian elbo} is equivalent to
	\[
	\displaystyle
	\maximize_{p \,\in\, \mathcal{P}}\; E(p; q_{\text{\normalfont mix}}^{(\lambda^\star)})
	\]
	which, together with Lemma~\ref{lem:equivalence}, completes the proof.
\end{proof}

\subsection{Proof of Theorem~\ref{thm: characterize dual solution} and Feasibility Criterion}\label{app:more theory}
We start with the proof of Theorem~\ref{thm: characterize dual solution}.
\begin{proof}
	Similar to the proof of Theorem~\ref{thm:solution of unparameterized constrained optimization}, we begin with the Lagrangian of Problem~\eqref{eq:ELBO-constrained unparameterized optimization}, 
	\begin{eqnarray}
		\mathcal{L}(p, \lambda) 
		\nonumber
		& = & \displaystyle
		-\, E(p; q) \,-\, \sum_{i\,=\,1}^m \lambda_i \left(E(p; q^i) + \bar b_i\right)
		\\[0.2cm]
		\label{eq: lambda0 and b0} 
		& = &\displaystyle \lambda^T \mathbf{1} \left( -\sum_{i \,=\, 1}^m \frac{\lambda_i}{\lambda^T \mathbf{1}} E(p;q^i)\right) \,-\, \lambda^T \bar b
		\\[0.2cm]
		\nonumber
		& = &\displaystyle \lambda^T \mathbf{1} \left(-\sum_{i \,=\, 1}^m\frac{\lambda_i}{\lambda^T \mathbf{1}}\mathbb{E}_{q^i(x_0)} \mathbb{E}_{q(x_{1:T}\,\vert\,x_0)} \log \frac{p(x_{0:T})}{q(x_{1:T}\,\vert\,x_0)}\right) \,-\, \lambda^T \bar b
		\\[0.2cm]
		\nonumber
		& = &\displaystyle \lambda^T \mathbf{1} \left(-\mathbb{E}_{q^{(\lambda)}(x_0)} \mathbb{E}_{q(x_{1:T}\,\vert\,x_0)} \log \frac{p(x_{0:T})}{q(x_{1:T}\,\vert\,x_0)}\right) \,-\, \lambda^T \bar b
		\\[0.2cm]
		\nonumber
		& = &
		\displaystyle -\,(\lambda^T \mathbf{1})  E(p;q^{(\lambda)}) \,-\, \lambda^T \bar b
	\end{eqnarray}
	where from~\eqref{eq: lambda0 and b0} onwards we use notation: $\lambda_0 = 1$, $\bar b_0 = 0$, $\lambda = \left[\lambda_0, \ldots, \lambda_m \right]^T,\ \bar b = \left[\bar b_0, \ldots, \bar b_m \right]^T$, and use $q^0$ to represent $q$, which will be used in the rest of proof. To formulate the dual problem, we check the minimum of the Lagrangian,
	\begin{eqnarray}
		g(\lambda) & \DefinedAs & \displaystyle
		\nonumber\minimize_{p \,\in\, \mathcal{P}}
		\;
		\mathcal{L}(p, \lambda)
		\\
		\nonumber
		& = &
		\displaystyle \minimize_{p \,\in\, \mathcal{P}}\; -\,(\lambda^T \mathbf{1})  E(p;q^{(\lambda)}) - \lambda^T \bar b
		\\[0.2cm]
		\label{eq:dual function3}
		& = &\displaystyle  -\, \lambda^T \bar b 
		\,+\,
		(\lambda^T \mathbf{1}) \ \minimize_{p \,\in\, \mathcal{P}} \; -E(p;q^{(\lambda)})
	\end{eqnarray}
	where the only term that depends on $p$ is the ELBO. Recall that:
	\[
	D_{\text{KL}} (q(x_{0:T})\,\Vert\, p(x_{0:T})) 
	\; = \;
	-\, E(p; q)
	\, + \,
	\mathbb{E}_{q(x_0)}
	\left[\,
	\log q(x_0) 
	\,\right].
	\]
	Since the minimum value of the KL divergence is zero (attained when $p = q$), the minimum of $-E(p; q)$ is likewise attained when $p = q$. Thus, it is straightforward that the minimum is equal to the entropy of the distribution $q$, denoted by $h(q) \DefinedAs -\mathbb{E}_{q(x_0)}\left[\,\log q(x_0) \,\right]$. With this in mind, from~\eqref{eq:dual function3} we have
	\[
	g(\lambda) 
	\;=\;
	- \,\lambda^T \bar b 
	\,+\,
	(\lambda^T \mathbf{1}) \ h(q^{(\lambda)}).
	\]
	Thus, the dual problem reads
	\[
	\begin{array}{rl}
		\displaystyle\maximize_{\lambda\,\geq\,0} & g(\lambda) \;\DefinedAs\; - \,\lambda^T \bar b 
		\,+\,
		(\lambda^T \mathbf{1}) \ h(q^{(\lambda)}).
		
	\end{array}
	\]
	We first reformulate the entropy of the mixture distribution $q^{(\lambda)}$, 
	\begin{eqnarray}
		\displaystyle
		\nonumber
		h(q^{(\lambda)}) 
		& = &\displaystyle   -\mathbb{E}_{q^{(\lambda)}(x_0)} \left[ \log q^{(\lambda)}(x_0) \right]
		\\[0.2cm]
		\label{eq:disjoint assumption 1}\displaystyle
		& = & 
		\displaystyle  -\int \sum_{i \,=\, 0}^m \frac{\lambda_i}{\lambda^T \mathbf{1}} q^i (x_0) \log\left(\sum_{i \,=\, 1}^m \frac{\lambda_i}{\lambda^T \mathbf{1}} q^i (x_0)\right) dx_0
		\\[0.2cm]
		\label{eq:disjoint assumption 2}
		& = & \displaystyle -\int \sum_{i \,=\, 0}^m \frac{\lambda_i}{\lambda^T \mathbf{1}} q^i (x_0) \log\left(\frac{\lambda_i}{\lambda^T \mathbf{1}} q^i (x_0)\right) dx_0
		\\[0.2cm]
		\nonumber
		& = & 
		\displaystyle - \sum_{i \,=\, 0}^m \frac{\lambda_i}{\lambda^T \mathbf{1}} \underbrace{\int   q^i (x_0) \log\left(q^i (x_0)\right) dx_0}_{\DefinedAs\,-\,h_i} 
		\, - \,
		\sum_{i \,=\, 0}^m \frac{\lambda_i}{\lambda^T \mathbf{1}} \log\left(\frac{\lambda_i}{\lambda^T \mathbf{1}}\right)
		\\[0.2cm]
		\nonumber
		& = & 
		\displaystyle  \sum_{i \,=\, 0}^m \frac{\lambda_i}{\lambda^T \mathbf{1}} h_i \,-\, \sum_{i \,=\, 0}^m \frac{\lambda_i}{\lambda^T \mathbf{1}} \log\left(\frac{\lambda_i}{\lambda^T \mathbf{1}}\right)
		\\[0.2cm]
		\nonumber
		& = & 
		\displaystyle  \sum_{i \,=\, 0}^m \frac{\lambda_i}{\lambda^T \mathbf{1}} h_i 
		\,-\,
		\sum_{i \,=\, 0}^m \frac{\lambda_i}{\lambda^T \mathbf{1}} \log (\lambda_i) 
		\,+\,
		\log(\lambda^T \mathbf{1})
	\end{eqnarray} 
	where going from~\eqref{eq:disjoint assumption 1} to~\eqref{eq:disjoint assumption 2} we utilize the assumption that the distributions $\{q^i\}_{i\,=\,0}^m$ have disjoint supports; see Remark~\ref{rem: disjoint assumption} on when this is the case.
	
	Now, we can compute the gradient of the dual function over $\lambda_i$, $i=1,\ldots,m$,
	\[
	\begin{array}{rcl}
		\displaystyle\frac{\partial}{\partial \lambda_i} \left( - \lambda^T \bar b \,+\,
		(\lambda^T \mathbf{1}) \ h(q^{(\lambda)}) \right) 
		& = &
		\displaystyle \frac{\partial}{\partial \lambda_i} \left(  - \lambda^T \bar b + \sum_{j \,=\, 0}^m\lambda_j h_j - \sum_{j \,=\, 0}^m\lambda_j \log \lambda_j  + (\lambda^T \mathbf{1})\log(\lambda^T \mathbf{1})\right)\\[0.2cm]
		& = & \displaystyle h_i - \bar b_i - \log \left( \frac{\lambda_i}{\lambda^T \mathbf{1}}\right).
	\end{array}
	\]
	Setting the gradient be zeros allows us to find the optimal dual variables $\lambda^\star$,
	\[
	h_i - \bar b_i - \log \left( \frac{\lambda^\star_i}{(\lambda^\star)^T \mathbf{1}}\right) 
	\; =\; 0 \ \ \ \text{for } i = 1,\ldots, m.
	\]
	Hence,
	\begin{equation}\label{eq: optimal dual solution}
		\frac{\lambda^\star_i}{(\lambda^\star)^T \mathbf{1}} 
		\; = \; 
		{\rm e}^{h_i \,-\, \bar b_i} \ \ \ \text{for } i = 1,\ldots, m.
	\end{equation}
	We clarify that in~\eqref{eq: optimal dual solution}, $\lambda^\star = \left[\lambda^\star_0, \ldots, \lambda^\star_m \right]^T$ with its first element being $\lambda^\star_0 = 1$. Finally, if we return back to notation $\lambda^\star = \left[\lambda^\star_1, \ldots, \lambda^\star_m \right]^T$, then,
	\[
	\frac{\lambda^\star_i}{1+(\lambda^\star)^T \mathbf{1}} 
	\; = \; 
	{\rm e}^{h_i \,-\, \bar b_i} \ \ \ \text{for } i = 1,\ldots, m
	\]
	which completes the proof.
\end{proof}

\begin{remark}\label{rem: disjoint assumption}
	We remark on the assumption of the distributions $\{q^i\}_{i\,=\,0}^m$ having disjoint supports. In the setting of adapting model to new data in Section~\ref{subsec:fair image generation}, this is a reasonable assumption. Since we often finetune a pre-trained diffusion model on new data not seen in the original pre-training dataset, the new data distribution and the pre-training data distribution have mostly disjoint supports. In the minority class setting in Section~\ref{subsec:fair image generation}, the constrained distributions $\{q^i\}_{i\,=\,1}^m$ and the objective distribution $q^0$ usually are not disjoint. However, since the distributions $\{q^i\}_{i\,=\,1}^m$ often are often restrictions of $q^0$ to subsets of the support of $q^0$, i.e., the minority classes, the derivation of optimal dual variables is similar to what we have provided in this section, so we omit the repeated details. Extending these results to cases where the distributions are neither disjoint nor restrictions of the objective distributions, is challenging and has been left to future work.
\end{remark}

To prove a feasibility criterion, we first show that the dual function is finite in Lemma~\ref{lem:finite dual}.

\begin{lemma}[Boundedness of the optimal dual function]\label{lem:finite dual}
	Let the differential entropy $h_i$ of each distribution $q^i$ be finite. Then, 
	the optimal value of the dual function $D^\star \DefinedAs \max_{\lambda \,\geq\, 0} g(\lambda)$ is finite if and only if
	\[
	\sum_{i \,=\, 1}^m {\rm e}^{h_i \,-\, \bar b_i} 
	\; < \; 1.
	\]
\end{lemma}
\begin{proof}
	($\Leftarrow$) From  $\sum_{i \,=\, 1}^m {\rm e}^{h_i - \bar b_i} < 1$, the otpimal dual variable $\lambda^\star$ given by~\eqref{eq: optimal dual solution} is finite. Thus, the optimal value of the dual function $g(\lambda^\star)$ becomes
	\[
	g(\lambda^\star) \;=\; 
	-\, (\lambda^\star)^T \bar b 
	\, + \,  
	\sum_{i \,=\, 0}^m\lambda^\star_i h_i 
	\, - \, 
	\sum_{i \,=\, 0}^m\lambda^\star_i \log \lambda^\star_i  \,+\, ((\lambda^\star)^T \mathbf{1})\log((\lambda^\star)^T \mathbf{1})
	\]
	and $\lambda^\star_i = {\rm e}^{h_i - \bar b_i} > 0$, and also $\{h_i\}_{i\,=\,1}^m$ are all finite. Therefore, $D^\star$ is finite.
	
	($\Rightarrow$) We prove it by contradiction. Assume $\sum_{i \,=\, 1}^m {\rm e}^{h_i - \bar b_i} = {\rm e}^\delta \geq 1$ for some $\delta \geq 0$. For any $\lambda \geq 0$, there exists a direction in which $g(\lambda)$ increases, i.e.,
	\begin{equation}\label{eq: positive grad}
		\frac{\partial g}{\partial \lambda_i}  \;=\;
		h_i - \bar b_i - \log \left( \frac{\lambda_i}{\lambda^T \mathbf{1}}\right) > \delta \;\; \exists \  i.
	\end{equation}
	To see~\eqref{eq: positive grad} by contradiction, we check that
	\begin{eqnarray}
		\nonumber
		h_i - \bar b_i - \log \left( \frac{\lambda_i}{\lambda^T \mathbf{1}}\right) 
		\;\leq\; 
		\delta \text{   \ \ for } i = 1, \ldots, m
		\\
		\nonumber
		\implies {\rm e}^{h_i \,-\, \bar b_i \,-\, \delta} 
		\; \leq \; \frac{\lambda_i}{\lambda^T \mathbf{1}}\text{   \ \ for } i = 1, \ldots, m
		\\
		\nonumber
		\implies \left(\sum_{i \,=\, 1}^m {\rm e}^{h_i \,-\, \bar b_i}\right) {\rm e}^{-\delta} 
		\;\leq\;
		\frac{\sum_{i \,=\, 1}^m\lambda_i}{\lambda^T \mathbf{1}} \; =\;
		\frac{\sum_{i \,=\, 1}^m\lambda_i}{1 + \sum_{i \,=\, 1}^m\lambda_i} 
		\; <\;
		1
		\\
		\nonumber
		\implies \sum_{i \,=\, 1}^m {\rm e}^{h_i \,-\, \bar b_i} 
		\;<\;
		{\rm e}^\delta
	\end{eqnarray}
	which contradicts our assumption that $\sum_{i \,=\, 1}^m {\rm e}^{h_i \,-\, \bar b_i} = {\rm e}^\delta$. By contradiction, we have~\eqref{eq: positive grad}. Furthermore,~\eqref{eq: positive grad} implies that $g(\lambda)$ is unbounded above, which contradicts the finiteness of $D^\star$. Therefore, we must have that $\sum_{i\,=\,1}^m{\rm e}^{h_i \,-\, \bar b_i} < 1$. 
\end{proof}

\begin{lemma}[Feasibility criterion]\label{lem:feasibility}
	Let the differential entropy $h_i$ of each distribution $q^i$ be finite. Suppose that there exists a feasible solution to Problem~\eqref{eq:ELBO-constrained unparameterized optimization} such that its objective function is bounded from below. Then,  Problem~\eqref{eq:ELBO-constrained unparameterized optimization} is feasible if and only if 
	\[
	\sum_{i \,=\, 1}^m {\rm e}^{h_i \,-\, \bar b_i} 
	\; < \;
	1.
	\]
\end{lemma}
\begin{proof}
	($\Rightarrow$) Since the primal problem is feasible, the optimal objective function $F^\star$ is bounded from below and it is attained at a feasible point. For the sake of contradiction, we assume $\sum_{i\,=\,1}^m{\rm e}^{h_i \,-\, \bar b_i} \geq 1$, which is equivalent to $D^\star = \infty$ according to Lemma~\ref{lem:finite dual}. However, this violates weak duality, i.e., $D^\star \leq F^\star$. By contradiction, we must have $\sum_{i\,=\,1}^m{\rm e}^{h_i \,-\, \bar b_i} < 1$.
	
	($\Leftarrow$)
	We consider a set $\mathcal{A}$,
	\[
	\mathcal{A} \;\DefinedAs\; 
	\left\{ (u_1, \ldots, u_m, t) \,|\, -E(p, q^i) -\bar b_i \leq u_i \text{ for } i = 1, \ldots, m~\text{ and } -E(p, q^0) \leq t  \text{ for } p \in \mathcal{P} \right\}.
	\]
	The set $\mathcal{A}$ is convex since it is the intersection of $m+1$ epigraphs of convex functions. We also introduce another convex set $\mathcal{B}$,
	\[
	\mathcal{B} 
	\;\DefinedAs\;
	\left\{ (0, \ldots, 0, t)\  |\  t \in \mathbb{R}\right\}.
	\]
	We utilise proof by contradiction. Assume that the primal problem is infeasible. Then there doesn't exist any $p \in \mathcal{P}$ such that $-E(p, q^i) -\bar b_i \leq 0$ for all $i = 1, \ldots, m$. Hence, $\mathcal{A}$ and $\mathcal{B}$ are two disjoint convex sets. From the separating hyperplane theorem, there exists a hyperplane that separates them, i.e., $\exists v \in \mathbb{R}^{m + 1}$ and $c \in \mathbb{R}$,
	\begin{equation}\label{eq: sep hyp A}
		x^Tv \geq c \; \text{ for all } x \in \mathcal{A}
	\end{equation}
	\begin{equation}\label{eq: sep hyp B}
		y^Tv \leq c\; \text{ for all } y \in \mathcal{B}.
	\end{equation}
	Let $v = \left[\, \lambda_1, \ldots \lambda_m, \gamma \,\right]^T$. Then,~\eqref{eq: sep hyp B} reduces to 
	\[
	y^Tv 
	\;=\;
	\lambda^T u 
	\,+\, 
	\gamma t 
	\; = \;
	\gamma t 
	\;\leq\;
	c\; \text{ for all } (u, t) \in \mathcal{B} 
	\;\;
	\Rightarrow 
	\;\;
	\gamma \;=\; 0.
	\]
	This is because $\gamma t \leq c$ for any $t \in \mathbb{R}. $ Note that $\gamma = 0$ means the separating hyperplane is vertical, i.e., being parallel to the $t$-axis. Furthermore, from~\eqref{eq: sep hyp A} we can write
	\[
	x^Tv 
	\; = \;
	\lambda^T u \,+\, \gamma t 
	\; = \;
	\lambda^T u 
	\; \geq \; 
	c\; \text{ for all } (u, t) \in \mathcal{A} 
	\;\; \Rightarrow \;\;
	\lambda  \;\geq\; 0.
	\]
	The above is true because the set of values that each $u_i$ can take in $\mathcal{A}$ is unbounded above. Thus, since $\lambda^T u \geq c$, necessarily every $\lambda_i$ has to be non-negative. Now, we consider 
	\begin{eqnarray}
		\nonumber
		g(\lambda) &=& 
		\inf_{(u,t) \,\in\, \mathcal{A}}\; t \,+\, \lambda^T u 
		\\[0.2cm]
		\nonumber
		\implies\;\;
		g(\alpha \lambda) &=& \inf_{(u,t) \,\in\, \mathcal{A}} t 
		\,+\, \alpha \lambda^T u\; \text{ for }\alpha \in \mathbb{R}_+
		\\[0.2cm]
		\nonumber
		\implies\;\; \lim_{\alpha \,\rightarrow\, \infty} \;g(\alpha \lambda) &=& \lim_{\alpha \,\rightarrow\, \infty} \inf_{(u,t) \,\in\, \mathcal{A}}\; t \,+\, \alpha \lambda^T u
		\\[0.2cm]
		\nonumber
		&= & \lim_{\alpha \,\rightarrow\, \infty}\; \alpha \left( \inf_{(u,t) \,\in\, \mathcal{A}} \lambda^T u \right)
		\\[0.2cm]
		\nonumber
		\label{eq: d_star is infinite} & \geq & \lim_{\alpha \,\rightarrow\, \infty} \alpha  c 
		\\[0.2cm]
		\nonumber
		& = &\infty
	\end{eqnarray}
	which shows that  $D^* = \infty$. This contradicts $D^\star$ being finite ($D^\star$ is finite due to $\sum_{i \,=\, 1}^m {\rm e}^{h_i - \bar b_i} < 1$ and Lemma~\ref{lem:finite dual}). Because of the contradiction, the primal problem has to be feasible.
\end{proof}

\subsection{Proof of Lemma~\ref{lem: bounded TV distance}}\label{app: bounded TV distance}

\begin{proof}
	The proof is an application of the convergence theory of DDPM~\cite[Theorem~3]{li2023towards}. We next check all assumptions of~\cite[Theorem~3]{li2023towards}. It is easy to see that we can cast $q_{\text{mix}}^{(\lambda)}$ as a target distribution of a diffusion model. By the definition,
	\[
	\begin{array}{rcl}
		\bar{\mathcal{L}}(\theta,\lambda)
		& = &  \displaystyle
		D_{\text{KL}}\left(q (x_{0:T}) \,\Vert\,  p_\theta (x_{0:T})\right) 
		\,+\,
		\sum_{i\,=\,1}^m \lambda_i \left(D_{\text{KL}}\left(q^i(x_{0:T})\,\Vert\,  p_\theta(x_{0:T})\right) - b_i\right)
		\\[0.2cm]
		& = &\displaystyle
		- E(p_\theta; q) \, +\, \mathbb{E}_{q(x_0)}[ \,\log q(x_0)\, ]
		\,+\,
		\sum_{i\,=\,1}^m \lambda_i \left(-E(  p_\theta; q^i) - \bar{b}_i\right).
	\end{array} 
	\]
	Thus, the partial minimization of $\bar{\mathcal{L}}(\theta,\lambda)$ over $\theta$ is equivalent to a weighted EBLO minimization,
	\[
	\minimize_{\theta\,\in\,\Theta} \;\; 
	-\, E(p_\theta; q) 
	\,-\,
	\sum_{i\,=\,1}^m \lambda_i \, E(  p_\theta; q^i)
	\]
	or, equivalently,    
	\begin{equation}\label{eq:EBLO mixed lambda}
		\minimize_{\theta\,\in\,\Theta} \;\; -\, E(p_\theta; q_{\text{\normalfont mix}}^{(\lambda)})
	\end{equation}
	where we normalize the weighted ELBO objective by introducing a mixed data distribution $q_{\text{mix}}^{(\lambda)}$. We note that $\bar p^\star(\lambda)$ is also a minimizer of Problem~\eqref{eq:EBLO mixed lambda}.
	
	On the other hand, using Problem~\eqref{eq:KL-constrained likelihood maximization parametrization Elbo MSE}, we can rewrite Problem~\eqref{eq:EBLO mixed lambda} as 
	\[
	\minimize_{\theta\,\in\,\Theta} \;\;
	\mathbb{E}_{ q_{\text{\normalfont mix}}^{(\lambda)}(x_0),\, t,\, x_t}
	\left[\,  \norm{\hat{s}_\theta(x_t, t) - \nabla \log q(x_t)}^2
	\,\right]
	\]
	which is equivalent to the score matching objective in DDPM. Therefore, the score matching assumption in~\cite[Assumption~1]{li2023towards} is satisfied with the error bound $\varepsilon_{\text{\normalfont score}}^2$ from Assumption~\ref{ass: small approx error for score}. Viewing Assumption~\ref{ass:bounded samples}, and using appropriate stepsize and variance, all assumptions in~\cite[Theorem~3]{li2023towards} are satisfied. Therefore, application of~\cite[Theorem~3]{li2023towards} completes the proof.
\end{proof}

\subsection{Proof of Lemma~\ref{lem:dual un/parametrized gap}}\label{app:dual un/parametrized gap}

\begin{lemma}\label{lem:TV distance}
	The \text{\normalfont TV} distance between two mixture data distributions $q_{\text{\normalfont mix}}^{(\lambda)}$, $ q_{\text{\normalfont mix}}^{(\lambda^\star)}$ is bounded by
	\[
	\text{\normalfont TV}\left(q_{\text{\normalfont mix}}^{(\lambda)}, q_{\text{\normalfont mix}}^{(\lambda^\star)} \right) 
	\; \leq \; 
	\left\Vert\lambda-\lambda^{ \star}\right\Vert_1.
	\]
\end{lemma}
\begin{proof}
	By the definition, 
	\[
	\begin{array}{rcl}
		\text{\normalfont TV}\left(q_{\text{mix}}^{(\lambda)}, q_{\text{mix}}^{(\lambda^\star)} \right)  
		&  = & \displaystyle \frac{1}{2}\int_{x_0} \left\vert
		\frac{q + \sum_{i\,=\,1}^m\lambda^{i} q^i }{ 1+\lambda^\top 1}  
		\, - \,
		\frac{ q + \sum_{i\,=\,1}^m\lambda^{i, \star} q^i }{ 1+(\lambda^\star)^\top 1} \right\vert
		\\[0.4cm]
		&  = & \displaystyle \frac{1}{2}\int_{x_0} \left\vert
		\frac{(1+(\lambda^\star)^\top 1)(q + \sum_{i\,=\,1}^m\lambda^{i} q^i ) - (1+\lambda^\top 1)(q + \sum_{i\,=\,1}^m\lambda^{i, \star} q^i)}{ (1+\lambda^\top 1) (1+(\lambda^\star)^\top 1)}  \right\vert
		\\[0.4cm]
		&  = & \displaystyle \frac{1}{2}\int_{x_0} \left\vert
		\frac{ \sum_{i\,=\,1}^m\lambda^{i} q^i + (\lambda^\star)^\top 1 q  - \sum_{i\,=\,1}^m\lambda^{i, \star} q^i - \lambda^\top 1 q }{ (1+\lambda^\top 1) (1+(\lambda^\star)^\top 1)}  \right\vert
		\\[0.4cm]
		&  \leq & \displaystyle \frac{1}{2}\int_{x_0} \left\vert
		\sum_{i\,=\,1}^m(\lambda^{i}-\lambda^{i, \star}) q^i + (\lambda^\star- \lambda)^\top 1 q \right\vert
		\\[0.4cm]
		&  \leq & \displaystyle 
		\sum_{i\,=\,1}^m\vert\lambda^{i}-\lambda^{i, \star}\vert
		\\[0.4cm]
		&  = & \displaystyle
		\left\Vert\lambda-\lambda^{ \star}\right\Vert_1
	\end{array}
	\]
	where the first inequality is due to ${ (1+\lambda^\top 1) (1+(\lambda^\star)^\top 1)} \geq 1$, and we use triangle inequality in the second inequality.
\end{proof}

We recall the Lagrangians for Problems~\eqref{eq:KL-constrained likelihood maximization unparametrization Elbo MSE} and~\eqref{eq:KL-constrained likelihood maximization parametrization Elbo MSE},
\[
\begin{array}{rcl}
	\!\!\!\!
	\!\!\!\!
	\!\!
	\mathcal{L}_s(\hat{s},\lambda)
	& = &\displaystyle
	\mathbb{E}_{q(x_0),\, t,\, x_t}
	\left[\,
	\norm{\hat{s}(x_t, t) - \nabla\log q(x_t)}^2
	\,\right] 
	\\[0.2cm]
	&& 
	\displaystyle + \,
	\sum_{i \, = \, 1}^m
	\lambda_i
	\left(
	\mathbb{E}_{q^i(x_0),\, t,\, x_t}
	\left[\,
	\norm{\hat{s}(x_t, t) - \nabla\log q(x_t)}^2
	\,\right] 
	- \Tilde{b}^i
	\right)
\end{array}
\]
\[
\bar{\mathcal{L}}_s(\hat{s}_\theta,\lambda)
\; = \;
{\mathcal{L}}_s(\hat{s}_\theta,\lambda)
\]
and their associated dual functions,
\[
g_s(\lambda)\; = \; \minimize_{\hat{s}\,\in\,\mathcal{S}} \; \mathcal{L}_s(\hat{s},\lambda)
\; 
\text{ and }
\;
\bar{g}_s(\lambda)\; = \; \minimize_{\theta \,\in\, \Theta} \; \bar{\mathcal{L}}_s(\hat{s}_\theta,\lambda).
\]
For brevity, we use shorthand $\mathbb{E}_{q}$ and $\mathbb{E}_{q^i}$ for $\mathbb{E}_{q(x_0),\, t,\, x_t}$ and $\mathbb{E}_{q^i(x_0),\, t,\, x_t}$, respectively. 

\begin{lemma}[Parametrization gap]\label{lem:dual parametrization gap}
	Let Assumption~\ref{ass:richness} hold. Then,
	$0\leq \bar{g}_s(\lambda) - g_s(\lambda) \leq 4R (1+\norm{\lambda}_1)\nu$ for any $\lambda\geq 0$. 
\end{lemma}
\begin{proof}
	Let the partial minimizer of $\mathcal{L}_s(\hat{s}, \lambda)$ over $\hat{s}$ be $\hat{s}^\star(\lambda) \DefinedAs \argmin_{\hat{s}} \mathcal{L}_s(\hat{s}, \lambda)$ for any $\lambda\geq0$.
	For any $\lambda\geq0$,
	there exists $\Tilde{\theta}\in\Theta$ such that    $\norm{\hat{s}^\star(\lambda) - \hat{s}_{\Tilde{\theta}}}_{L_2} \leq \nu$ for any $\lambda\geq0$, according to Assumption~\ref{ass:richness}.
	Thus,
	\[
	\begin{array}{rcl}
		\bar{\mathcal{L}}_s(\hat{s}_{\Tilde{\theta}},\lambda) - \mathcal{L}_s(\hat{s}^\star(\lambda),\lambda)
		& = & \displaystyle
		\mathbb{E}_{q}
		\left[\,
		\norm{\hat{s}_{\Tilde{\theta}}(x_t, t) - x_0}^2
		\,\right] 
		\, - \, 
		\mathbb{E}_{q}
		\left[\,
		\norm{\hat{s}^\star(\lambda)(x_t, t) - x_0}^2
		\,\right]
		\\[0.2cm]
		&  & \displaystyle
		+ \,
		\sum_{i \, = \, 1}^m
		\lambda_i
		\left(
		\mathbb{E}_{q^i}
		\left[\,
		\norm{\hat{s}_{\Tilde{\theta}}(x_t, t) - x_0}^2
		\,\right] 
		- 
		\mathbb{E}_{q^i}
		\left[\,
		\norm{\hat{s}^\star(\lambda)(x_t, t) - x_0}^2
		\,\right]
		\right)
		\\[0.4cm]
		& \leq & \displaystyle
		4R\, \mathbb{E}_{q}
		\left[\,
		\norm{\hat{s}_{\Tilde{\theta}}(x_t, t) -\hat{s}^\star(\lambda)(x_t, t) }
		\,\right]
		\\[0.2cm]
		&  & \displaystyle
		+\, 4R
		\sum_{i \, = \, 1}^m
		\lambda_i\,
		\mathbb{E}_{q^i}
		\left[\,
		\norm{\hat{s}_{\Tilde{\theta}}(x_t, t) - \hat{s}^\star(\lambda)(x_t, t) }
		\,\right]
		\\[0.4cm]
		& \leq & 4R \nu \,+\, 4R \norm{\lambda}_1 \nu
	\end{array}
	\]
	where the first inequality is due to that the quadratic function is locally Lipschitz continuous with parameter $4R$, and the second inequality is because 
	that there exists $\Tilde{\theta}\in\Theta$ such that    $\norm{\hat{x}^\star(\lambda) - \hat{x}_{\Tilde{\theta}}}_{L_2} \leq \nu$ for any $\lambda\geq0$, according to Assumption~\ref{ass:richness}.
	
	By the definition $\hat{s}_\theta^\star(\lambda) \in \argmin_{\theta\,\in\,\Theta} \bar{\mathcal{L}}_s(\hat{s}_\theta, \lambda)$,
	\[
	\bar{\mathcal{L}}_s(\hat{s}^\star_\theta(\lambda), \lambda)
	\; \leq\; \bar{\mathcal{L}}_s(\hat{s}_{\Tilde{\theta}}, \lambda).
	\]
	Therefore,
	\[
	0\;\leq\; {\mathcal{L}}_s(\hat{s}^\star_\theta(\lambda), \lambda)
	- {\mathcal{L}}_s(\hat{s}^\star(\lambda), \lambda) 
	\;\leq\;
	\bar{\mathcal{L}}_s(\hat{s}_{\Tilde{\theta}}, \lambda)- {\mathcal{L}}_s(\hat{s}^\star(\lambda), \lambda) 
	\;\leq\;
	4R (1+\norm{\lambda}_1)\nu
	\]
	which gives our desired result by the definition of dual functions.
\end{proof}

\begin{lemma}[Differentiability]\label{lem:dual differentiability}
	The dual function $g_s(\lambda)$ is differentiable with gradient $\nabla_\lambda \mathcal{L}_s(\hat{s}^\star(\lambda),\lambda)$.
\end{lemma}
\begin{proof}
	For any $\lambda\geq 0$, the Lagrangian $\mathcal{L}_s(\hat{s},\lambda)$ is strongly convex in function $\hat{s} \in \mathcal{S}$. Since $\mathcal{S}$ is convex and compact, any partial minimizer $\hat{s}^\star(\lambda)$ is unique. By Danskin's theorem~\cite{bertsekas2016nonlinear}, $g_s(\lambda)$ is differentiable and its gradient is the gradient of $\mathcal{L}_s(\hat{s},\lambda)$ over $\lambda$ at $\hat{s} = \hat{s}^\star(\lambda)$.  
\end{proof}


\begin{lemma}[Convexity]\label{lem:dual convexity}
	The dual function $g_s(\lambda)$ is
	$\mu$-strongly concave in $\lambda \in \mathcal{H}$, where 
	\[
	\mu 
	\;=\; \left(\frac{\sigma}{1+\max\left(\norm{\lambda^\star}_1, \norm{\bar\lambda^\star}_1\right)}\right)^2.
	\]
\end{lemma}
\begin{proof}
	For any $\lambda_1$, $\lambda_2 \in \mathcal{H}$, we denote  $\hat{s}_1^\star \DefinedAs \hat{s}^\star(\lambda_1)$ and $\hat{s}_2^\star \DefinedAs \hat{s}^\star(\lambda_2)$, which are unique partial minimizers of the Lagrangians $\mathcal{L}_s(\hat{s},\lambda_1)$ and $\mathcal{L}_s(\hat{s},\lambda_2)$. Denote $\ell_0(\hat{s}) \DefinedAs \mathbb{E}_{q} \big[\, \norm{\hat{s}(x_t, t) - \nabla\log q(x_t)}^2\,\big]$, $\ell_i(\hat{s}) \DefinedAs \mathbb{E}_{q^i} \big[\, \norm{\hat{s}(x_t, t) - \nabla\log q(x_t)}^2\,\big]$ for $i = 1,\ldots,m$, and $\ell(\hat{s}) \DefinedAs  [\,\ell_1(\hat{s}), \ldots, \ell_m(\hat{s})\,]^\top$. By the convexity of $\ell_i$, 
	\[
	\ell_i(\hat{s}_1^\star)
	\; \geq \; \displaystyle
	\ell_i(\hat{s}_2^\star)
	\,+\,
	2 \left\langle\nabla_{\hat{s}} \ell_i(\hat{s}_2^\star) ,
	\hat{s}_1^\star - \hat{s}_2^\star
	\right\rangle
	\] 
	\[
	\ell_i(\hat{s}_2^\star)
	\; \geq \; \displaystyle
	\ell_i(\hat{s}_1^\star)
	\,+\,
	2 \left\langle\nabla_{\hat{s}} \ell_i(\hat{s}_1^\star) ,
	\hat{s}_2^\star - \hat{s}_1^\star
	\right\rangle. 
	\] 
	If we multiply the first inequality above by $\lambda_{2,i}\geq0$ and the second inequality above by $\lambda_{1,i}\geq0$, and add them up from both sides, then,  
	\[
	\displaystyle- \left\langle\ell(\hat{s}_2) - \ell(\hat{s}_1), \lambda_{2} - \lambda_{1}\right\rangle
	\;  \geq  \; 
	2\left\langle
	\lambda_1^\top \nabla_{\hat{s}} \ell (\hat{s}_1^\star) - \lambda_2^\top \nabla_{\hat{s}} \ell (\hat{s}_2^\star),
	\hat{s}_2^\star - \hat{s}_1^\star
	\right\rangle.
	\]
	We apply Lemma~\ref{lem:dual differentiability} to the LHS of the inequality above,
	\begin{equation}\label{eq:dual cross}
		\displaystyle-\,\left(\nabla g_s(\lambda_2)
		-
		\nabla g_s(\lambda_1)
		\right)^\top (\lambda_{2} - \lambda_{1})
		\; \geq  \;
		2\left\langle
		\lambda_1^\top \nabla_{\hat{s}} \ell (\hat{s}_1^\star) - \lambda_2^\top \nabla_{\hat{s}} \ell (\hat{s}_2^\star),
		\hat{s}_2^\star - \hat{s}_1^\star
		\right\rangle.
	\end{equation}
	On the other hand, by the optimality of $\hat{s}_1^\star$ and $\hat{s}_2^\star$,
	\begin{subequations}\label{eq:two point optimality conditions}
		\begin{equation}
			\nabla_{\hat{s}} \ell_0(\hat{s}_1^\star) 
			\,+\, 
			\lambda_{1}^\top \nabla_{\hat{s}} \ell(\hat{s}_1^\star) 
			\; = \;
			0
		\end{equation}
		\begin{equation}
			\nabla_{\hat{s}} \ell_0(\hat{s}_2^\star) 
			\,+\, 
			\lambda_{2}^\top \nabla_{\hat{s}} \ell(\hat{s}_2^\star) 
			\; = \;
			0
		\end{equation}
	\end{subequations}
	which allows us to simplify the right-hand side of~\eqref{eq:dual cross} and obtain
	\begin{equation}\label{eq:dual cross new}
		\begin{array}{rcl}
			\displaystyle-\, \left(\nabla g_s(\lambda_2)
			-
			\nabla g_s(\lambda_1)
			\right)^\top (\lambda_{2} - \lambda_{1})
			&  \geq  & 
			\displaystyle
			2\left\langle
			\nabla_{\hat{s}} \ell_0(\hat{s}_2^\star) - \nabla_{\hat{s}} \ell_0(\hat{s}_1^\star),
			\hat{s}_2^\star - \hat{s}_1^\star
			\right\rangle
			\\[0.2cm]
			&  \geq  & 
			\displaystyle 2 \norm{ \hat{s}_1^\star-\hat{s}_2^\star}_{L_2}^2
		\end{array}
	\end{equation}
	where the last inequality results from the strong convexity of quadratic functionals.
	
	By the smoothness of quadratic functionals with parameter $1$,
	\[
	\begin{array}{rcl}
		\norm{ \hat{s}_1^\star-\hat{s}_2^\star}_{L_2} 
		& \geq &
		\norm{ \nabla_{\hat{s}} \ell_0(\hat{s}_1^\star) - \nabla_{\hat{s}} \ell_0(\hat{s}_2^\star)}_{L_2} 
		\\[0.2cm]
		& = & \displaystyle
		\norm{ \lambda_1^\top \nabla_{\hat{s}} \ell(\hat{s}_1^\star) - \lambda_2^\top \nabla_{\hat{s}} \ell(\hat{s}_2^\star)}_{L_2} 
		\\[0.2cm]
		& = & \displaystyle
		\norm{ (\lambda_2-\lambda_1)^\top \nabla_{\hat{s}} \ell(\hat{s}_2^\star) - \lambda_1^\top (\nabla_{\hat{s}} \ell(\hat{s}_1^\star)-\nabla_{\hat{s}} \ell(\hat{s}_2^\star))}_{L_2}
		\\[0.2cm]
		& \geq & \displaystyle
		\norm{ (\lambda_2-\lambda_1)^\top \nabla_{\hat{s}} \ell(\hat{s}_2^\star)}_{L_2}
		-
		\norm{ \lambda_1^\top (\nabla_{\hat{s}} \ell(\hat{s}_1^\star)
			\, - \,
			\nabla_{\hat{s}} \ell(\hat{s}_2^\star))}_{L_2}
	\end{array}
	\]
	where the equality is due to the optimality condition~\eqref{eq:two point optimality conditions} and the last inequality is due to triangle inequality.
	By Assumption~\ref{ass:redundancy},
	\[
	\displaystyle
	\norm{ (\lambda_2-\lambda_1)^\top \nabla_{\hat{s}} \ell(\hat{s}_2^\star)}_{L_2}
	\;\geq\; 
	\sigma \norm{\lambda_2 - \lambda_1}.
	\]
	We also notice that 
	\[
	\begin{array}{rcl}
		\norm{ \lambda_1^\top (\nabla_{\hat{s}} \ell(\hat{s}_1^\star)-\nabla_{\hat{s}} \ell(\hat{s}_2^\star))}_{L_2} 
		&  \leq &  \displaystyle\sum_{i\,=\,1}^m \lambda_{1,i}\norm{  \nabla_{\hat{s}} \ell_i(\hat{s}_1^\star)-\nabla_{\hat{s}} \ell_i(\hat{s}_2^\star)}_{L_2}
		\\[0.2cm]
		&  \leq &  \displaystyle\sum_{i\,=\,1}^m \lambda_{1,i}\norm{  \hat{s}_1^\star-\hat{s}_2^\star}_{L_2}
	\end{array}
	\]
	where the first inequality is due to triangle inequality and the second inequality is due to the smoothness of quadratic functionals.
	Hence,
	\[
	\norm{ \hat{s}_1^\star-\hat{s}_2^\star}_{L_2}
	\; \geq\;
	\sigma\norm{\lambda_2-\lambda_1} 
	\,-\, 
	\norm{\lambda_1}_1 \norm{ \hat{s}_1^\star -\hat{s}_2^\star}_{L_2}
	\]
	or, equivalently,
	\[
	\norm{ \hat{s}_1^\star-\hat{s}_2^\star}_{L_2}
	\;\geq\;
	\frac{\sigma}{1+\norm{\lambda_1}_1}\norm{\lambda_2-\lambda_1}
	\]
	Therefore,~\eqref{eq:dual cross new} becomes
	\[
	\displaystyle-\left(\nabla g_s(\lambda_2)
	-
	\nabla g_s(\lambda_1)
	\right)^\top (\lambda_{2} - \lambda_{1})
	\;\geq\;
	\left(\frac{\sigma}{1+\norm{\lambda_1}_1}\right)^2
	\norm{\lambda_2-\lambda_1}^2
	\]
	which completes the proof by choosing the smallest modulus over $\lambda_1\in\mathcal{H}$.
\end{proof}

\begin{proof}
	By Lemmas~\ref{lem:dual differentiability} and~\ref{lem:dual convexity}, for any $\lambda\in\mathcal{H}$,
	\[
	g_s(\lambda) 
	\; \leq\;
	g_s(\lambda^\star)
	\,+\, 
	\nabla g_s(\lambda^\star)^\top (\lambda-\lambda^\star)
	\,-\,
	\frac{\mu}{2} \norm{\lambda- \lambda^\star}^2.
	\]
	Thus, if we choose $\lambda = \bar\lambda^\star$, then
	\[
	g_s(\bar\lambda^\star) 
	\; \leq\;
	g_s(\lambda^\star)
	\,+\, 
	\sum_{i \, = \, 1}^m
	(\bar\lambda_i^\star-\lambda_i^\star)
	\left(
	\mathbb{E}_{q^i}
	\left[\,
	\norm{\hat{s}^\star(\lambda^\star)(x_t, t) - \nabla \log q(x_t)}^2
	\,\right] 
	- \Tilde{b}^i
	\right)
	\,-\,
	\frac{\mu}{2} \norm{\bar\lambda^\star- \lambda^\star}^2.
	\]
	Optimality of $(\hat{s}^\star(\lambda^\star), \lambda^\star)$ leads to the complementary slackness,
	\[
	\sum_{i \, = \, 1}^m
	\lambda_i^\star
	\left(
	\mathbb{E}_{q^i}
	\left[\,
	\norm{\hat{s}^\star(\lambda^\star)(x_t, t) - \nabla \log q(x_t)}^2
	\,\right] 
	- \Tilde{b}^i
	\right) \; = \; 0
	\]
	and the feasibility,
	\[
	\mathbb{E}_{q^i}
	\left[\,
	\norm{\hat{s}^\star(\lambda^\star)(x_t, t) - \nabla \log q(x_t)}^2
	\,\right] 
	\; \leq \;
	\Tilde{b}^i.
	\]
	Therefore,
	\[
	g_s(\bar\lambda^\star) 
	\; \leq\;
	g_s(\lambda^\star)
	\, - \,
	\frac{\mu}{2} \norm{\bar\lambda^\star- \lambda^\star}^2.
	\]
	According to Lemma~\ref{lem:dual parametrization gap}, $\bar{g}_s(\bar\lambda^\star) - 4R \left(1+\norm{\bar\lambda^\star}_1\right) \nu \leq g_s(\bar\lambda^\star)$. Hence,
	\[
	\bar{g}_s(\bar\lambda^\star)
	\, - \, 
	4R\left(1+\norm{\bar\lambda^\star}_1\right) \nu 
	\; \leq \;
	g_s(\lambda^\star) 
	\, - \,
	\frac{\mu}{2} \norm{\bar{\lambda}^\star - \lambda^\star}^2.
	\]
	Thus,
	\[
	\begin{array}{rcl}
		\norm{\bar{\lambda}^\star - \lambda^\star}^2 
		& \leq & \displaystyle
		\frac{2}{\mu}\left(g_s(\lambda^\star) -
		\bar{g}_s(\bar\lambda^\star)\right)
		\, + \, 
		\frac{8}{\mu}R \left(1+\norm{\bar\lambda^\star}_1\right) \nu 
		\\[0.3cm]
		& \leq & \displaystyle
		\frac{8}{\mu}R \left(1+\norm{\bar\lambda^\star}_1\right) \nu 
	\end{array}
	\]
	where the last inequality is due to that $g_s(\lambda) \leq \bar{g}_s(\lambda)$ for any $\lambda\geq 0$, and the optimality of $\bar{\lambda}^\star$,
	\[
	g_s(\lambda^\star) 
	\;\leq\; 
	\bar{g}_s(\lambda^\star)
	\; \leq \;
	\bar{g}_s(\bar\lambda^\star).
	\]
\end{proof}

\subsection{Proof of Theorem~\ref{thm: bounded TV distance ideal}}\label{app: bounded TV distance ideal}

\begin{proof}
	By the triangle inequality for TV distance,
	\[
	\begin{array}{rcl}
		\text{\normalfont TV}\left(q_{\text{\normalfont mix}}^{\star},\; \bar p^\star(\bar\lambda^\star)\right)
		& \leq & \displaystyle
		\text{\normalfont TV}\left(q_{\text{\normalfont mix}}^{\star},\; q_{\text{\normalfont mix}}^{(\bar\lambda^\star)}\right)
		\,+\,
		\text{\normalfont TV}\left(q_{\text{\normalfont mix}}^{(\bar\lambda^\star)},\; \bar p^\star(\bar\lambda^\star)\right) 
		\\[0.2cm]
		& \leq & \displaystyle
		\norm{\lambda^\star - \bar\lambda^\star}_1
		\,+\,
		\frac{d^2 \log^3T}{\sqrt{T}} 
		\,+\, 
		\sqrt{d} \left(\log^2T\right) \varepsilon_{\text{\normalfont score}}
		\\[0.2cm]
		& \leq & \displaystyle
		\sqrt{\frac{8}{\mu}mR \left(1+\norm{\bar\lambda^\star}_1\right) \nu}
		\,+\,
		\frac{d^2 \log^3T}{\sqrt{T}} 
		\,+\, 
		\sqrt{d} \left(\log^2T\right) \varepsilon_{\text{\normalfont score}}
	\end{array}
	\]
	where the second inequality is due to Lemma~\ref{lem:TV distance} and Lemma~\ref{lem: bounded TV distance}, and the last inequality is due to $\norm{\lambda}_1 \leq \sqrt{m} \norm{\lambda}$ and Lemma~\ref{lem:dual un/parametrized gap}.
\end{proof}

\subsection{Proof of Theorem~\ref{thm:algorithm output vs unparametrized solution}}\label{app:algorithm output vs unparametrized solution}

\begin{lemma}\label{lem:dual iterate gap parametrization}
	For a stochastic variant of Algorithm~\ref{alg:primal-dual} in Section~\ref{subsec:primal-dual training}, we have 
	\[
	\mathbb{E} \left[\,
	\norm{\lambda(h+1) - \bar\lambda^\star}^2
	\,\big\vert\, \lambda(h)
	\,\right]
	\;\leq\;
	\norm{\lambda(h) - \bar\lambda^\star}^2
	\,+\,
	\eta^2 S^2
	\,-\,
	2\eta\left( \bar{D}^\star - \bar{g}(\lambda(h)) - \varepsilon_{\text{\normalfont approx}}^2\right).
	\]
\end{lemma}
\begin{proof}
	For brevity, we let the stochastic gradient be $\hat{f}(h) = [\hat{f}_1(h),\ldots, \hat{f}_m(h)]$ with
	\[
	\hat{f}_i(h)
	\; \DefinedAs\;
	\hat{\mathbb{E}}_{x_0 \,\sim\, q^i,\, t,\, x_t}
	\left[ \,
	\norm{\hat{s}_\theta(h)(x_t, t) - \nabla q(x_t)}^2
	\,\right] \,-\, \Tilde{b}^i.
	\]
	By the definition of $\lambda(h+1)$,
	\[
	\begin{array}{rcl}
		\norm{\lambda(h+1) - \bar{\lambda}^\star}^2 
		& = & \displaystyle
		\norm{ \left[\lambda(h) + \eta \hat{f}(h)\right]_+ - \bar{\lambda}^\star }^2
		\\[0.2cm]
		& \leq  & \displaystyle
		\norm{ \lambda(h)  - \bar{\lambda}^\star + \eta \hat{f}(h)}^2
		\\[0.2cm]
		& =  & \displaystyle
		\norm{ \lambda(h)  - \bar{\lambda}^\star }^2
		\,+\, 
		\eta^2 \norm{\hat{f}(h)}^2 \,+\,
		2\eta \hat{f}(h)^\top \left(\lambda(h)  
		-
		\bar{\lambda}^\star\right) 
	\end{array}
	\]
	where the inequality is due to the non-expansiveness of projection. Application of the conditional expectation over both sides of the inequality above yields,
	\[
	\begin{array}{rcl}
		\mathbb{E}\left[\,
		\norm{\lambda(h+1) - \bar{\lambda}^\star}^2 \,\vert\, \lambda(h)
		\,\right]
		& \leq  & \displaystyle
		\norm{ \lambda(h)  - \bar{\lambda}^\star }^2
		\,+\, 
		\eta^2\, \mathbb{E}\left[ \,\norm{\hat{f}(h)}^2 \,\vert\,\lambda(h)\,\right] 
		\\[0.2cm]
		&& \displaystyle+\,
		2\eta\, \mathbb{E}\left[\,\hat{f}(h)\,\vert\,\lambda(h)\,\right]^\top \left(\lambda(h)  - \bar{\lambda}^\star\right) 
	\end{array}
	\]
	which gives our desired result when we use the fact that $\mathbb{E}[\hat{f}(h)\,\vert\,\lambda(h)]$ is an approximate descent direction of the dual function $\bar{g}$,
	\[
	\mathbb{E}\left[\,\hat{f}(h)\,\vert\,\lambda(h)\,\right]^\top
	\left( \lambda(h) -\bar\lambda^\star\right) 
	\,-\, \varepsilon_{\text{\normalfont approx}}^2
	\;\leq\;
	\bar{g}(\lambda(h)) 
	\,-\, 
	\bar{g}(\bar\lambda^\star).
	\]
\end{proof} 

\begin{lemma}\label{lem:best dual iterate}
	In the stochastic variant of Algorithm~\ref{alg:primal-dual} in Section~\ref{subsec:primal-dual training}, the maximum prarametrized dual function in history up to step $h$ satisfies
	\[
	\lim_{h\,\to\,\infty} \;
	\bar{g}_{\text{\normalfont best}}(h) 
	\; \geq \;
	\bar{D}^\star \,-\, \left( \frac{\eta S^2}{2} + \varepsilon_{\text{\normalfont approx}}^2 \right).
	\]
\end{lemma}
\begin{proof}
	The proof is based on the supermartingale convergence theorem~\cite[Theorem E7.4]{solo1994adaptive}. We introduce two processes,
	\[
	\alpha(h) 
	\; \DefinedAs\; 
	\norm{\lambda(h) - \bar\lambda^\star}^2 \mathbbm{1}\left(\bar{D}^\star - \bar{g}_{\text{best}}(h) > \frac{\eta S^2}{2} + \varepsilon_{\text{\normalfont approx}}^2\right)
	\]
	\[
	\beta(h) 
	\; \DefinedAs\; 
	\left( 
	2\eta\left(\bar{D}^\star - \bar{g}(\lambda(h))-\varepsilon_{\text{\normalfont approx}}^2\right)
	-\eta^2 S^2
	\right) \mathbbm{1}\left(\bar{D}^\star - \bar{g}_{\text{best}}(h) > \frac{\eta S^2}{2} + \varepsilon_{\text{\normalfont approx}}^2\right)
	\]
	where $\alpha(h)$ measures the gap between $\lambda(h)$ and $\bar{\lambda}^\star$ when the optimality gap $\bar{D}^\star - \bar{g}_{\text{best}}(h)$ is below the threshold, and $\beta(h)$ measures the gap bewteen $\bar{D}^\star$ and $\bar{g}(\lambda(h))$ (up to some optimization errors) when the optimality gap $\bar{D}^\star - \bar{g}_{\text{best}}(h)$ is below the threshold. Clearly, $\alpha(h)$ is non-negative. Also, $\beta(h)$ is non-negative due to
	\[
	\displaystyle
	\bar{D}^\star 
	\,-\, 
	\bar{g}_{\text{best}}(h) 
	\,-\, 
	\frac{\eta S^2}{2} 
	\,-\, \varepsilon_{\text{\normalfont approx}}^2
	\; \leq \;
	\bar{D}^\star 
	\,-\, 
	\bar{g}(\lambda(h)) 
	\,-\, 
	\frac{\eta S^2}{2} 
	\,-\, \varepsilon_{\text{\normalfont approx}}^2.
	\]
	
	Let $\mathcal{F}_h$ be the $\sigma$-algebra generated by sequences: $\alpha(h')$, $\beta(h')$, and $\lambda(h')$ for $h' \leq h$. Thus, $\{ \mathcal{F}_h \}_{h\,\geq\,1}$ is a natural filtration. We notice that $\alpha(h+1)$ and $\beta(h+1)$ are determined by $\lambda(h)$ in each step. Hence,
	\[
	\begin{array}{rcl}
		\mathbb{E} 
		\left[ \, 
		\alpha(h+1) 
		\,\vert\,
		\mathcal{F}_h
		\, \right]
		& = & \displaystyle
		\mathbb{E}
		\left[ \, 
		\alpha(h+1) \,\vert\, \lambda(h)
		\,\right] 
		\\[0.2cm]
		& = & \displaystyle
		\mathbb{E}
		\left[ \,
		\alpha(h+1) \,\vert\, \lambda(h), \alpha(h) = 0
		\,\right] \, \text{Pr}\left(\alpha(h) = 0\right)
		\\[0.2cm]
		&  & \displaystyle +\,
		\mathbb{E}
		\left[ \,
		\alpha(h+1) \,\vert\, \lambda(h), \alpha(h) > 0
		\,\right] \text{Pr}\left(\alpha(h) > 0\right).
	\end{array}
	\]
	We first show that
	\begin{equation}\label{eq:key objective of lemma 8}
		\mathbb{E}\left[\,
		\alpha(h+1) \,\vert\, \mathcal{F}_h
		\,\right]
		\; \leq \;
		\alpha(h) \,-\, \beta(h).
	\end{equation}
	A simple case is when $\alpha(h) = 0$,
	\[\,
	\mathbb{E} 
	\left[ 
	\alpha(h+1) 
	\,\vert\,
	\mathcal{F}_h
	\,\right]
	\; = \;
	\mathbb{E}
	\left[\, 
	\alpha(h+1) \,\vert\, \lambda(h), \alpha(h) = 0
	\,\right]. 
	\]
	There are two situations for $\alpha(h) = 0$. First, if $\bar{D}^\star - \bar{g}_{\text{best}}(h) \leq \frac{\eta S^2}{2} + \varepsilon_{\text{\normalfont approx}}^2$, then $\alpha(h) = \beta(h) = 0$. Due to $\bar{g}_{\text{best}}(h+1)\geq \bar{g}_{\text{best}}(h)$, we have $\beta(h) = 0$ and $\bar{D}^\star - \bar{g}_{\text{best}}(h+1) \leq \frac{\eta S^2}{2} + \varepsilon_{\text{\normalfont approx}}^2$. Thus, $\alpha(h+1) = 0$, and~\eqref{eq:key objective of lemma 8} holds. Second, if $\lambda(h) = \bar{\lambda}^\star$, but $\bar{D}^\star - \bar{g}_{\text{best}}(h) > \frac{\eta S^2}{2} + \varepsilon_{\text{\normalfont approx}}^2$, then $\bar{D}^\star = \bar{g}(\lambda(h))$. Hence, $\beta(h)<0$, which contradicts the non-negativeness of $\beta(h)$. Therefore, $\bar{D}^\star - \bar{g}_{\text{best}}(h) \leq \frac{\eta S^2}{2} + \varepsilon_{\text{\normalfont approx}}^2$ has to hold, which is the first situation.
	
	We next show~\eqref{eq:key objective of lemma 8} when $\alpha(h)>0$.
	\[
	\begin{array}{rcl}
		\mathbb{E} 
		\left[\, 
		\alpha(h+1) 
		\,\vert\,
		\mathcal{F}_h
		\,\right]
		& = & \displaystyle
		\mathbb{E}
		\left[\, 
		\alpha(h+1) \,\vert\, \lambda(h), \alpha(h) > 0
		\,\right] 
		\\[0.2cm]
		& = & \displaystyle
		\mathbb{E}
		\left[\, 
		\norm{\lambda(h) - \bar\lambda^\star}^2 \mathbbm{1}\left(\bar{D}^\star - \bar{g}_{\text{best}}(h) > \frac{\eta S^2}{2} + \varepsilon_{\text{\normalfont approx}}^2\right)
		\,\Bigg\vert\, \lambda(h), \alpha(h) > 0
		\,\right] 
		\\[0.4cm]
		& \leq & \displaystyle
		\mathbb{E}
		\left[\, 
		\norm{\lambda(h) - \bar\lambda^\star}^2
		\,\big\vert\, \lambda(h), \alpha(h) > 0
		\,\right]
		\\[0.4cm]
		& \leq & \displaystyle
		\norm{\lambda(h) - \bar\lambda^\star}^2
		\,+\,
		\eta^2 S^2
		\,-\,
		2\eta\left( \bar{D}^\star - \bar{g}(\lambda(h)) - \varepsilon_{\text{\normalfont approx}}^2\right)
		\\[0.4cm]
		& \leq & \displaystyle
		\alpha(h) \,-\, \beta(h)
	\end{array}
	\]
	where the last inequality is due to Lemma~\ref{lem:dual iterate gap parametrization}, and the last equality is due to $\bar{D}^\star - \bar{g}_{\text{best}}(h) > \frac{\eta S^2}{2} + \varepsilon_{\text{\normalfont approx}}^2$. Therefore,~\eqref{eq:key objective of lemma 8} holds.
	
	Finally, application of the supermartingale convergence theorem~\cite[Theorem E7.4]{solo1994adaptive} to the processes $\alpha(h)$ and $\beta(h)$ for $h\geq 1$ concludes that $\beta(t)$ is almost surely summable,
	\[
	\liminf_{h\,\to\,\infty} \; \beta(h) \; = \;0 \;  \text{ almost surely}.
	\]
	This means that either
	\[
	\liminf_{t\,\to\,\infty}\; 
	2\eta\left(\bar{D}^\star - \bar{g}(\lambda(h))-\varepsilon_{\text{\normalfont approx}}^2\right)
	\,-\,
	\eta^2 S^2
	\; = \;
	0
	\]
	or $\bar{D}^\star - \bar{g}_{\text{best}}(h) \leq \frac{\eta S^2}{2} + \varepsilon_{\text{\normalfont approx}}^2$, which concludes our desired result.
\end{proof}

Denote the step when $\bar{g}_{\text{best}}(h)$ achieves $\bar{D}^\star - \left( \frac{\eta S^2}{2} + \varepsilon_{\text{\normalfont approx}}^2 \right)$ by $h_{\text{best}}$, and the associated dual variable be $\bar\lambda_{\text{best}} = \lambda(h_{\text{best}})$.

\begin{lemma}\label{lem:dual un/parametrized gap best}
	For a stochastic variant of Algorithm~\ref{alg:primal-dual} in Section~\ref{subsec:primal-dual training}, we have
	\[
	\norm{\bar\lambda_{\text{\normalfont best}} - \lambda^\star}^2
	\; \leq \;
	\displaystyle
	\frac{2}{\mu}
	\left(
	\frac{\eta S^2}{2} + \varepsilon_{\text{\normalfont approx}}^2 
	+
	4R (1+\norm{\bar\lambda_{\text{\normalfont best}}}_1)\nu
	\right).
	\]
\end{lemma}
\begin{proof}
	We denote the segment between $\bar\lambda_{\text{best}}$ and $\lambda^\star$ by $\mathcal{B}$. By Lemma~\ref{lem:dual convexity}, the dual function $g(\lambda)$ is strongly concave on $\mathcal{B}$ with parameter $\mu$. Thus,
	\[
	\begin{array}{rcl}
		\norm{\bar\lambda_{\text{best}} - \lambda^\star}^2
		&  \leq  & \displaystyle
		\frac{2}{\mu}
		\left(
		g(\lambda^\star) - g(\bar\lambda_{\text{best}})
		\right)
		\\[0.2cm]
		& \leq &
		\displaystyle
		\frac{2}{\mu}
		\left(
		\bar g(\lambda^\star) - \bar{g}(\bar\lambda_{\text{best}})+4R (1+\norm{\bar\lambda_{\text{best}}}_1)\nu
		\right)
		\\[0.2cm]
		& \leq &
		\displaystyle
		\frac{2}{\mu}
		\left(
		\frac{\eta S^2}{2} + \varepsilon_{\text{\normalfont approx}}^2 
		+
		4R (1+\norm{\bar\lambda_{\text{best}}}_1)\nu
		\right)
	\end{array}
	\]
	where the second inequality is due to Lemma~\ref{lem:dual parametrization gap}. We note that $\bar{g}(\bar\lambda_{\text{best}}) = \bar{g}(\lambda(h_{\text{best}}))$ and $\bar{D}^\star \geq \bar{g}(\bar\lambda_{\text{best}})$, and the third inequality is due to  Lemma~\ref{lem:best dual iterate}. 
\end{proof}

\begin{proof}
	By the triangle inequality for TV distance,
	\[
	\begin{array}{rcl}
		\text{\normalfont TV}\left(q_{\text{\normalfont mix}}^{\star},\; \bar p^\star(\bar\lambda_{\text{\normalfont best}})\right)
		& \leq & \displaystyle
		\text{\normalfont TV}\left(q_{\text{\normalfont mix}}^{\star},\; q_{\text{\normalfont mix}}^{(\bar\lambda_{\text{\normalfont best}})}\right)
		\,+\,
		\text{\normalfont TV}\left(q_{\text{\normalfont mix}}^{(\bar\lambda_{\text{\normalfont best}})},\; \bar p^\star(\bar\lambda_{\text{\normalfont best}})\right) 
		\\[0.2cm]
		& \leq & \displaystyle
		\norm{\lambda^\star - \bar\lambda_{\text{\normalfont best}}}_1
		\,+\,
		\frac{d^2 \log^3T}{\sqrt{T}} 
		\,+\,
		\sqrt{d} \left(\log^2T\right) \varepsilon_{\text{\normalfont score}}
		\\[0.2cm]
		& \leq & \displaystyle
		\frac{2}{\mu}
		\left(
		\frac{\eta S^2}{2} + \varepsilon_{\text{\normalfont approx}}^2 
		+
		4R \left(1+\norm{\bar\lambda_{\text{\normalfont best}}}_1\right)\nu
		\right)
		\,+\,
		\frac{d^2 \log^3T}{\sqrt{T}} 
		\,+\, 
		\sqrt{d} \left(\log^2T\right) \varepsilon_{\text{\normalfont score}}
	\end{array}
	\]
	where the second inequality is due to Lemma~\ref{lem:TV distance} and Lemma~\ref{lem: bounded TV distance}, and the last inequality is due to $\norm{\lambda}_1 \leq \sqrt{m} \norm{\lambda}$ and Lemma~\ref{lem:dual un/parametrized gap best}.
\end{proof}

\section{Experimental details}\label{app:implementation}
	
We provide implementation details of our computational experiments in Section~\ref{sec:experiments}. The source code is available here.\footnote[2]{ \href{https://github.com/shervinkhal/Constrained_Diffusion_Dual_Training}{https://github.com/shervinkhal/Constrained\_Diffusion\_Dual\_Training}}

\noindent\textbf{Algorithm details: }We train our constrained diffusion models by replacing the exact primal minimization step in Algorithm~\ref{alg:primal-dual} with $N$ steps of gradient descent with the Lagrangian as a loss function. Without loss of generality, we take the noise prediction formulation of diffusion rather than the score-matching formulation used in our theory. Since these two formulations are equivalent, this has no bearing on our main results. Algorithm~\ref{alg:primal-dual-practical} depicts our practical implementation of Algorithm~\ref{alg:primal-dual} .

\begin{algorithm}
	\caption{Practical Implementation of Algorithm~\ref{alg:primal-dual}}\label{alg:primal-dual-practical}
	\begin{algorithmic}[1]
		\State \textbf{Input}: total diffusion steps $T$, diffusion parameter $\alpha_t$, total dual iterations $H$, number of primal descent steps per dual update $N$, dual step size $\eta_d$, primal step size $\eta_p$, initial model parameters $\theta(0)$.
		\State \textbf{Initialize}: $\lambda(1) = 0$.
		\For{$h = 1, \cdots, H$}
		\For{$n = 1, \cdots, N$}
		\State $\theta_1 \; = \; \theta{(h - 1)}.$
		
		\State $\displaystyle\theta_{n + 1} 
		\; = \; 
		\theta_n 
		\, - \,
		\eta_p \, \nabla_\theta \left( \hat{\mathbb{E}}_{x_0 \,\sim\, q, \,t, \,x_t}
		\left[ \,
		\norm{\hat{\epsilon}_{\theta_n}(x_t, t) - \epsilon_0}^2
		\,\right]  
		\, + \,
		\sum_{i \,=\, 1}^m \lambda_i \, \hat{\mathbb{E}}_{x_0 \,\sim\, q^i,\, t,\, x_t}
		\left[ \,
		\norm{\hat{\epsilon}_{\theta_n}(x_t, t) - \epsilon_0}^2
		\,\right] \right).$
		
		\State $\theta(h) \; = \; \theta_{N + 1}$.
		\EndFor
		\State Update the dual variable
		\[
		\lambda_i (h + 1) 
		\; = \;
		\left[\, 
		\lambda_i(h) 
		\, + \,
		\eta_d \left( \hat{\mathbb{E}}_{x_0 \,\sim\, q^i, \,t, \,x_t}
		\left[\, 
		\norm{\hat{\epsilon}_{\theta(h)}(x_t, t) - \epsilon_0}^2
		\,\right]  
		\,-\, 
		\Tilde{b}^i \right) 
		\,\right]_+ \text{ for all } i. 
		\]
		\EndFor
	\end{algorithmic}
\end{algorithm}

In Algorithm~\ref{alg:primal-dual-practical}, the unbiased estimate of the noise prediction loss is evaluated via
\[
\hat{\mathbb{E}}_{x_0 \,\sim\, q, \,t, \,x_t}
\left[ \, 
\norm{\hat{\epsilon}_{\theta}(x_t, t) - \epsilon_0}^2
\,\right] 
\; = \;
\sum_{i \;=\; 1}^B \norm{\hat{\epsilon}_{\theta}(x_{t^{(i)}}, t^{(i)}) - \epsilon_0^{(i)}}^2
\]
where $\{x_0^{(i)}\}_{i\,=\,1}^B$ is a randomly chosen batch of samples from $q$, $\{t^{(i)}\}_{i\,=\,1}^B$ are time steps randomly sampled from the interval $\left[2, T\right]$, and $x_{t^{(i)}}$ is the noisy version of $x_0^{(i)}$ at time step $t^{(i)}$. Then, the noise $\epsilon_0$ is sampled from the standard Gaussian, and $x_t$  is derived from $\displaystyle x_t = \sqrt{\Bar{\alpha}_t} x_t + \sqrt{1 - \Bar{\alpha}_t}\epsilon_0$.

We remark an important implementation detail in the fine-tuning experiment. In the fine-tuning constraints, we have to evaluate the KL divergence
\begin{equation}\label{eq:1}
	D_{\text{KL}}(p_{\theta_{\text{pre}}}(x_{0:T})\,\Vert\, p_\theta(x_{0:T}))\; \leq \; b_i    
\end{equation}
where $p_{\theta_{\text{pre}}}(x_{0:T})$ is the joint distribution of the samples and latents generated by the backward process of the pre-trained model. The KL divergence constraint in~\eqref{eq:1} further reduces to
\begin{equation}\label{eq:2}
	D_{\text{KL}}(p_{\theta_{\text{pre}}}(x_{0:T})\,\Vert\, p_\theta(x_{0:T})) 
	\;=\; 
	{\mathbb{E}}_{x_t \,\sim\, p_{\theta_{\text{pre}}},\, t}
	\left[ \,
	\norm{\hat{\epsilon}_{\theta_{\text{pre}}}(x_t, t) - \hat{\epsilon}_{\theta}(x_t, t)}^2
	\,\right] 
	\,+\, \text{constant}.
\end{equation}
To estimate the expectation in~\eqref{eq:2}, we need to sample latents $x_t$ from the backward distribution $p_{\theta_{\text{pre}}}(x_{0:T})$. In practice, this is computationally inefficient, since it requires running inference each time one wants to sample a latent. This is why we implement this with sampling $x_t$ as random Gaussian noise instead. This still ensures that the predictions of the new model $p_\theta$ don't differ too much from the pre-trained distribution $p_{\theta_{\text{pre}}}$ while making sampling batches much faster.

\noindent\textbf{Resilient constrained learning.}
The choice of the constraint thresholds $\{b_i\}_{i\,=\,1}^m$ has noticeable effect on the training of constrained diffusion models. To avoid an exhaustive hyperparameter tuning process, in the minority class experiment, we use the resilient constrained learning technique~\cite{hounie2024resilient} to adjust the thresholds $\{b_i\}_{i\,=\,1}^m$ during training. In essence, the resilient constrained learning adds a constraint relaxation cost to the loss and relaxes the thresholds by updating them through gradient descent each time we update the dual variable.  It can further be shown theoretically that an equivalent formulation of resilience is achieved by adding a quadratic regularizer of the dual variable into the loss objective and setting the constraint thresholds to be zero, i.e., $\Tilde{b}_i = 0$. This is the approach we used in our experiments since it has fewer hyperparameters. We note that the only difference between Algorithm~\ref{alg:primal-dual-practical} and Algorithm~\ref{alg:resilient} is the additional term in the dual variable update step (line 9 of Algorithm~\ref{alg:resilient}).

\begin{algorithm}
	\caption{Resilient Constrained Diffusion Models via Dual Training}\label{alg:resilient}
	\begin{algorithmic}[1]
		\State \textbf{Input}: total diffusion steps $T$, diffusion parameter $\alpha_t$, total dual iterations $H$, number of primal descent steps per dual update $N$, dual step size $\eta_d$, primal step size $\eta_p$, constraint step size $\eta_c$, initial model parameters $\theta(0)$, constraint relaxation cost $\gamma$.
		\State \textbf{Initialize}: $\lambda(1) = 0$.
		\For{$h = 1, \cdots, H$}
		\For{$n = 1, \cdots, N$}
		\State $\theta_1 \; = \; \theta{(h - 1)}$.
		
		\State $\displaystyle\theta_{n + 1} 
		\; = \; 
		\theta_n 
		\,-\,
		\eta_p \, \nabla_\theta \left( \hat{\mathbb{E}}_{x_0 \,\sim\, q,\, t,\, x_t}
		\left[\, 
		\norm{\hat{\epsilon}_{\theta_n}(x_t, t) - \epsilon_0}^2
		\,\right]  
		\,+\,
		\sum_{i = 1}^m \lambda_i \hat{\mathbb{E}}_{x_0 \,\sim\, q^i, \,t, \,x_t}
		\left[\,
		\norm{\hat{\epsilon}_{\theta_n}(x_t, t) - \epsilon_0}^2
		\,\right] \right)$.
		
		\State $\theta(h) 
		\;=\; 
		\theta_{N + 1}$.
		\EndFor
		\State Update the dual variable
		\[
		\lambda_i (h + 1) 
		\; = \;
		\left[\, \lambda_i(h) 
		\, + \,
		\eta_d\, \left( \hat{\mathbb{E}}_{x_0 \,\sim\, q^i, \,t, \,x_t}
		\left[ \, 
		\norm{\hat{\epsilon}_{\theta(h)}(x_t, t) - \epsilon_0}^2
		\,\right]  
		\,-\, 
		\Tilde{b}^i(h) 
		\,-\,
		2 \gamma \lambda_i(h) \right) \,\right]_+ \text{ for all } i.
		\]
		\EndFor
	\end{algorithmic}
\end{algorithm}

\noindent\textbf{Model architecture.} We use a time-conditioned U-net model as is common in image diffusion tasks for all three datasets. The time conditioning is done by adding a positional embedding of the time to the input image. The parameters of the model are summarized in Table~\ref{tab:modelparams}.
\renewcommand{\arraystretch}{1.5}
\begin{table}
	\centering
	\caption{Parameters of U-net Model used as noise predictor}
	\begin{tabular}{ c c c c }
		\hline\hline
		\# Res-Net layers per U-Net block & 2  \\ \hline
		\# Res-Net down/upsampling blocks & 6  \\ \hline
		\# Output channels for U-Net blocks & (128, 128, 256, 256, 512, 512)  \\ \hline\hline
	\end{tabular}
	\label{tab:modelparams}
\end{table}
The fifth downsampling block and the corresponding upsampling block are Res-Net blocks with spatial self-attention.

\noindent\textbf{Hyperparameters.} We summarize the important hyperparameters in our experiments in Table~\ref{tab:hyperparams}. In the unconstrained models that we train for comparison, we use the same hyperparameters as the constrained version, disregarding the parameters related to the dual and relaxation updates. For models trained on Image-Net, when training the constrained models, we initialized to the parameters of the unconstrained model to make training times shorter.

\noindent \textbf{Hyperparameter sensitivity.} We remark the  sensitivity of the dual training algorithm to the number of dual iterations, primal/dual batch sizes, and primal/dual learning rates.
\begin{itemize}
	\item \textbf{Number of dual iterations:} In our implementation this shows up as the number of primal GD steps per dual update, $N$. Experimentally, we have observed that as long as $N$ is greater than 1, the results are not sensitive to this value. Additionally, the dual updates add a negligible computational overhead. Hence, updating the dual nearly as many times as we update model parameters doesn't reduce training efficiency.
	
	\item \textbf{Primal/dual batch sizes:} We have included results of training a constrained model on an unbalanced subset of MNIST, using different primal/dual batch sizes (See Table~\ref{tab:batch}.) The results suggest that for the minority class experiments, when the ratio between Primal and Dual batch sizes is larger, the model performs better (lower FID and more evenly distributed samples). This is in line with the heuristic we used in the included experiments in the paper where we chose the batch sizes such that the ratio of primal to dual batch size is close to size ratio of entire dataset to constraint datasets (which are much smaller). Howver for the fine-tuning task, the batch sizes did not seem to affect the final result as much.
	
	\item \textbf{Primal/dual learning rate:} For the primal learning rate, we followed the best practice used to train standard diffusion models. For the dual learning rate $\eta$
	, we refer to Theorem 8 in the paper, showing a smaller error bound for smaller 
	$\eta$ while slowing convergence. In practice, as long as 
	$\eta \leq 1$, we observed that the model converges to similar results reliably.
\end{itemize}

\noindent \textbf{Efficiency of constrained diffusion:} We note that the complexity of sampling from our constrained diffusion model does not increase with the number of constraints, as our trained diffusion model functions like a standard diffusion model to generate samples. Importantly, we remark that training our constrained diffusion model has comparable efficiency to training standard diffusion models detailed next.

The additional computational cost of our dual-based training (Algorithm~\ref{alg:primal-dual}) arises from: (i) updating the dual variables; (ii) updating the diffusion model in the primal update.
\begin{itemize}
	\item \textbf{Cost of updating the dual variables: } We note that our dual-based training has the same number of dual variables as the number of constraints. Thus, the cost for the dual update is linear in the number of constraints. To update each dual variable, we can directly use the ELBO loss over the batches sampled from each constrained dataset (already computed for the Lagrangian). Therefore, the cost of updating dual variables is negligible.
	
	\item \textbf{Cost of updating the diffusion model in the primal update: } We note that the primal update trains a standard diffusion model based on the Lagrangian with updated dual variables. In our experiments, this primal training often requires as few as 2-3 updates per dual update. Thus, when training our constrained model, we can train for the same number of epochs as an unconstrained model but update the dual variables after every few primal steps. As a result, training our constrained diffusion model is almost as efficient as training standard unconstrained models.
\end{itemize}

The only concern we encountered regarding efficiency is that batches need to be sampled from every constrained dataset at each step to estimate the Lagrangian. This introduces a small GPU memory overhead that increases with additional constraints. However, this is somewhat mitigated by the fact that constrained datasets are often much smaller than the original dataset, allowing us to choose a smaller batch size for the constrained datasets without degrading performance (see discussion on batch sizes in hyperparameter sensitivity section).

\renewcommand{\arraystretch}{1.5}
\begin{table}
	\centering
	\caption{Hyperparameter values used in the main experiments. MC denotes Minority Class experiments and FT denotes Fine-Tuning experiments. - denotes that resilience was not used for the experiment.}
	\begin{tabular}{ c c c c c}
		\hline\hline
		& MNIST MC & Celeb-A MC & Image-Net MC & MNIST FT \\ \hline 
		\#training epochs ($N\times H$) & 250 & 1000 & 2000 &500 \\ \hline
		\# primal steps per dual step ($N$) & 2 & 2 &2& 2\\ \hline
		Primal batch size & 128 & 256 &128& 256 \\ \hline
		Dual batch size & 64 & 128 & 64&256 \\ \hline
		Primal learning rate & 0.0001 & 0.0001 &5e-5 &5e-6 \\ \hline
		Dual learning rate & 0.1 & 0.1 & 0.05&1e5 \\ \hline
		Resilience Relaxation cost & 0.09 & 0.005 &0.025& - \\ \hline
		main dataset size & 31000 & 12500 &2000& 200 \\ \hline
		constraint dataset(s) size & 5000 & 500 &64& - \\ \hline\hline
	\end{tabular}
	\label{tab:hyperparams}
\end{table}


\renewcommand{\arraystretch}{1.5}
\begin{table}
	\centering
	\caption{Constrained model trained on MNIST with different Primal/Dual Batch sizes}
	\begin{tabular}{ c c c c c }
		\\ \hline\hline 
		(Primal batch size, Dual batch size) & (64, 16) & (64, 64) & (128, 16)& (128, 64) \\ \hline 
		FID score & \textbf{16.6} & 20.6 & 16.7& 20.24\\ \hline\hline  \\
	\end{tabular}
	\label{tab:batch}
\end{table}

\noindent\textbf{FID scores.} As a quantitative means of evaluating our constrained diffusion models, we use the FID (Frechet Inception Distance) as a metric to gauge the quality of the samples generated by diffusion models. The FID score was first introduced in~\cite{heusel2017gans} in form of
\begin{equation}
	d_{\text{FID}}^2((m, C), (m_w, C_w)) 
	\; = \;
	\norm{m - m_w}_2^2 
	\, + \,
	\text{Tr}(C + C_w - 2(CC_w)^{1/2})
\end{equation}
where $m$ and $C$ represent mean and variance, respectively, of the distribution of the features of the data samples which have been extracted by an inception model~\cite{szegedy2015going}. Similarly, $m_w$ and $C_w$ represent the mean and variance of some reference distribution that we are computing the distance to. In our experiment, we compute the FID scores by generating $15000$ samples from the diffusion model we are evaluating and comparing them to a balanced version of the original dataset. In the experiment with MNIST, this is the actual dataset. In the experiments with Celeb-A and Image-Net, since there is an imbalance in the original dataset, we consider a balanced subset of each with an equal number of samples from each class as reference. We use the clean-FID library~\cite{parmar2022aliased} for standard computation of the FID scores.

We note that our FID scores are somewhat larger compared to typical baselines in the literature. This is expected as a consequence of our experimental setup. We train both the unconstrained and constrained models, on a biased subset of the dataset wherein some of the classes have significantly fewer samples than the rest. We then compute the FID scores for these models compared to the actual dataset itself which is unbiased (i.e., every class has the same number of samples). These FID scores approximate how close the learned distribution of the model trained on biased data, is to the underlying unbiased distribution.

This setup contrasts with existing results in the literature, where the FID is computed with respect to unbiased data, and the models are also trained on unbiased data. Therefore, it is expected that such models will achieve better FID scores than constrained or unconstrained models trained with biased data. Our purpose in reporting the FIDs was not to compare them to existing results (as such a comparison would be uninformative) but to demonstrate that, when trained on biased data, the constrained model achieves better FID scores than the unconstrained model.

\noindent\textbf{Compute resources.} We run all experiments on two NVIDIA RTX 3090-Ti GPUs in parallel.  The amount of GPU memory used was 16 Gigabytes per GPU. For experiments with MNIST and Celeb-A datasets, training each model took between 2-3 hours. This increased to 7-8 hours for latent diffusion models trained for the Image-Net experiments.

\noindent\textbf{Assets and libraries.} We use the PyTorch~\cite{paszke2019pytorch} and Diffusers~\cite{von-platen-etal-2022-diffusers} Python libraries for training our constrained diffusion models, and Adam with decoupled weight decay~\cite{loshchilov2019decoupled} as an optimizer. The accelerate library~\cite{accelerate} is used for the parallelization of the training processes across multiple GPUs. For classifiers used in evaluating the generated samples of the models, we use the following pretrained models accessible on the Huggingface model database: A Vision transformer-based classifier for MNIST digits with $\%99.5$ validation accuracy (\url{https://huggingface.co/farleyknight-org-username/vit-base-mnist}). A classifier for images of male/female faces with $\%98.6$ validation accuracy (\url{https://huggingface.co/cledoux42/GenderNew\_v002}). For classifying the image-net data, a zero-shot classifier based on a CLIP model (\url{https://huggingface.co/openai/clip-vit-base-patch32}) was used. The Autoencoder for the latent diffusion model was the stable diffusion VAE with KL regularization found on (\url{https://huggingface.co/stabilityai/sd-vae-ft-mse}).

\newpage
\section*{NeurIPS Paper Checklist}

\begin{enumerate}
	
	\item {\bf Claims}
	\item[] Question: Do the main claims made in the abstract and introduction accurately reflect the paper's contributions and scope?
	\item[] Answer: \answerYes{} %
	\item[] Justification: We summarize our constrained diffusion models in Section~\ref{sec:constrained diffusion models}, optimality guarantees of unparametrized constrained diffusion models in Section~\ref{subsec:unparametrized case}, optimality guarantees of parametrized constrained diffusion models and training algorithms in Section~\ref{sec:dual training algorithm}, and experimental results in Section~\ref{sec:experiments}.
	
	\item {\bf Limitations}
	\item[] Question: Does the paper discuss the limitations of the work performed by the authors?
	\item[] Answer: \answerYes{} 
	\item[] Justification: We mention two potential limitations of this work as several future directions in  Section~\ref{sec:conclusion}.

	\item {\bf Theory Assumptions and Proofs}
	\item[] Question: For each theoretical result, does the paper provide the full set of assumptions and a complete (and correct) proof?
	\item[] Answer: \answerYes{} 
	\item[] Justification: We provide all assumptions, lemmas, and theorems in Sections~\ref{sec:constrained diffusion models} and~\ref{sec:dual training algorithm}, and provide proof details in Appendix~\ref{app:proofs}. 
	
	\item {\bf Experimental Result Reproducibility}
	\item[] Question: Does the paper fully disclose all the information needed to reproduce the main experimental results of the paper to the extent that it affects the main claims and/or conclusions of the paper (regardless of whether the code and data are provided or not)?
	\item[] Answer: \answerYes{} 
	\item[] Justification: We summarize our experimental results in Section~\ref{sec:experiments}, and provide implementation details of experiments in Appendix~\ref{app:implementation}, together with additional experimental results.  

	\item {\bf Open access to data and code}
	\item[] Question: Does the paper provide open access to the data and code, with sufficient instructions to faithfully reproduce the main experimental results, as described in supplemental material?
	\item[] Answer: \answerYes{} 
	\item[] Justification: A link to the code for replicating our main experiments has been provided in Appendix~\ref{app:implementation}. The datasets are open source machine learning datasets that are accessible online.

	\item {\bf Experimental Setting/Details}
	\item[] Question: Does the paper specify all the training and test details (e.g., data splits, hyperparameters, how they were chosen, type of optimizer, etc.) necessary to understand the results?
	\item[] Answer: \answerYes{} 
	\item[] Justification: We provide experimental details in Appendix~\ref{app:implementation} including the hyperparameters used for each experiment. Our training details can be found in the code in supplemental material. 
	
	\item {\bf Experiment Statistical Significance}
	\item[] Question: Does the paper report error bars suitably and correctly defined or other appropriate information about the statistical significance of the experiments?
	\item[] Answer: \answerNA{} 
	\item[] Justification: While the histogram plots that showcase our main results do not include error bars, we do provide Frechet Distance metrics (that are a measure of statistical distance between sample distributions) which emphasize our results.
	
	\item {\bf Experiments Compute Resources}
	\item[] Question: For each experiment, does the paper provide sufficient information on the computer resources (type of compute workers, memory, time of execution) needed to reproduce the experiments?
	\item[] Answer: \answerYes{} 
	\item[] Justification: We include compute details in Appendix~\ref{app:implementation}.
	
	\item {\bf Code Of Ethics}
	\item[] Question: Does the research conducted in the paper conform, in every respect, with the NeurIPS Code of Ethics \url{https://neurips.cc/public/EthicsGuidelines}?
	\item[] Answer: \answerYes{} 
	\item[] Justification: We have reviewed the NeurIPS Code of Ethics and fully comply with it during the preparation of this paper. 

	\item {\bf Broader Impacts}
	\item[] Question: Does the paper discuss both potential positive societal impacts and negative societal impacts of the work performed?
	\item[] Answer: \answerYes{} 
	\item[] Justification: The potential impacts of the work are discussed in Section~\ref{sec: introduction}.
	
	\item {\bf Safeguards}
	\item[] Question: Does the paper describe safeguards that have been put in place for responsible release of data or models that have a high risk for misuse (e.g., pretrained language models, image generators, or scraped datasets)?
	\item[] Answer: \answerNA{} 
	\item[] Justification: We release no models and the supplemental code for training image generation model, trains the model on MNIST and Celeb-A datasets neither of which have potential for misuse.
	
	\item {\bf Licenses for existing assets}
	\item[] Question: Are the creators or original owners of assets (e.g., code, data, models), used in the paper, properly credited and are the license and terms of use explicitly mentioned and properly respected?
	\item[] Answer: \answerYes{}{} 
	\item[] Justification: The libraries and assets used have been noted and credited in Appendix~\ref{app:implementation}.
	
	\item {\bf New Assets}
	\item[] Question: Are new assets introduced in the paper well documented and is the documentation provided alongside the assets?
	\item[] Answer: \answerNA{} 
	\item[] Justification: \answerNA{}
	
	\item {\bf Crowdsourcing and Research with Human Subjects}
	\item[] Question: For crowdsourcing experiments and research with human subjects, does the paper include the full text of instructions given to participants and screenshots, if applicable, as well as details about compensation (if any)? 
	\item[] Answer: \answerNA{} 
	\item[] Justification: \answerNA{}
	
	\item {\bf Institutional Review Board (IRB) Approvals or Equivalent for Research with Human Subjects}
	\item[] Question: Does the paper describe potential risks incurred by study participants, whether such risks were disclosed to the subjects, and whether Institutional Review Board (IRB) approvals (or an equivalent approval/review based on the requirements of your country or institution) were obtained?
	\item[] Answer: \answerNA{} 
	\item[] Justification: \answerNA{}
	
\end{enumerate}


\end{document}